\documentclass[11pt]{article}
\usepackage[english]{babel}
\usepackage[utf8]{inputenc}
\usepackage{amsmath,amsfonts,amssymb,amsthm}
\usepackage[letterpaper,top=2cm,bottom=2cm,left=3cm,right=3cm,marginparwidth=1.75cm]{geometry}
\usepackage{graphicx}
\usepackage[colorlinks=true, allcolors=blue]{hyperref}
\usepackage{enumitem}
\usepackage{bm}
\usepackage{cleveref}
\usepackage[numbers,sort&compress]{natbib}
\usepackage[table]{xcolor}
\newtheorem{theorem}{Theorem}[section]
\newtheorem{corollary}[theorem]{Corollary}

\newtheorem{lemma}[theorem]{Lemma}
\newtheorem{proposition}[theorem]{Proposition}

\theoremstyle{definition}
\newtheorem{definition}[theorem]{Definition}

\newtheorem{assumption}[theorem]{Assumption}

\usepackage{authblk}
\usepackage{mathrsfs}
\usepackage[dvipsnames]{xcolor}
\usepackage{cancel}
\usepackage{tikz}
\usetikzlibrary{shapes, positioning, calc}
\usetikzlibrary{positioning, arrows.meta}
\usepackage{subcaption}
\usepackage{lineno}

\title{On the Unique Recovery of Transport Maps\\and Vector Fields from Finite Measure-Valued Data}
\author[1]{Jonah Botvinick-Greenhouse}
\author[2]{Yunan Yang}
\affil[1]{Center for Applied Mathematics, Cornell University, Ithaca, NY}
\affil[2]{Department of Mathematics, Cornell University, Ithaca, NY}

\begin{document}

\maketitle 

\begin{abstract}
We establish guarantees for the unique recovery of vector fields and transport maps from finite measure-valued data, yielding new insights into generative models, data-driven dynamical systems, and PDE inverse problems. In particular, we  provide general conditions under which a diffeomorphism can be uniquely identified from its pushforward action on \textit{finitely} many densities, i.e., when the data $\{(\rho_j,f_\#\rho_j)\}_{j=1}^m$ uniquely determines $f$. As a corollary, we introduce a new metric which compares diffeomorphisms by measuring the discrepancy between finitely many pushforward densities in the space of probability measures.  We also prove analogous results in an infinitesimal setting, where derivatives of the densities along a smooth  vector field are observed, i.e., when $\{(\rho_j,\textup{div} (\rho_j v))\}_{j=1}^m$ uniquely determines $v$. Our analysis makes use of the Whitney and Takens embedding theorems, which provide estimates on the required number of densities $m$, depending only on the intrinsic dimension of the problem.  We additionally interpret our results through the lens of Perron--Frobenius and Koopman operators and demonstrate how our techniques lead to new guarantees for the well-posedness of certain PDE inverse problems related to continuity, advection, Fokker--Planck, and advection-diffusion-reaction equations. Finally, we present illustrative numerical experiments demonstrating the unique identification of transport maps from finitely many pushforward densities, and of vector fields from finitely many weighted divergence observations. 
\end{abstract}

\section{Introduction}\label{sec:intro}

The problem of estimating transport maps and vector fields from measure-valued data is pervasive across data science, machine learning, and engineering. Such problems arise prominently in modern sampling and generative modeling frameworks that aim to transform a reference noise distribution into a target data distribution, including approaches based on optimal transport, normalizing flows, and diffusion processes \cite{ho2020denoising,lipman2022flow,papamakarios2021normalizing}. Beyond machine learning, measure-valued datasets are intrinsic to many physical and biological settings, where experimental observations are naturally represented as time-indexed empirical distributions rather than labeled particle trajectories. In these contexts, a common goal is to infer transport maps, vector fields, or dynamical laws from the evolution of distributions, in order to recover interpretable physical or biological structure \cite{tong2020trajectorynet,huguet2022manifold,schiebinger2019optimal,zhang2024learning,zhao2025cyclegrn,yao2025learning,scarvelis2023riemannian}.

Closely related problems arise in inverse formulations of partial differential equations (PDEs) governed by continuity, Fokker--Planck, or advection-diffusion equations, where one seeks to recover an underlying vector field or drift from measure-valued solution data \cite{liu2025inversions,blickhan2025dice,liu2021solving,chen2021solving}. Similar questions also appear in data-driven dynamical systems, where the objective is to identify Perron--Frobenius operators from their action on a finite collection of probability measures \cite{mcdonald2022identification}. More broadly, \textit{regression problems over spaces of probability measures}, where inputs and outputs are related by pushforward or transport operators, have recently attracted significant attention \cite{chen2023wasserstein,botvinick2025measure}.

Despite the rapid growth of these applications, fundamental questions of well-posedness and identifiability remain poorly understood. In particular, it is often unclear whether a transport map or vector field can be uniquely determined from its action on a finite number of probability measures, even in idealized noiseless settings. The goal of this paper is to provide foundational results addressing this gap. We establish practical conditions under which a transport map or vector field is uniquely identifiable from finitely many measure-valued observations. Our results yield new identifiability guarantees that apply across data-driven dynamical systems, PDE inverse problems, and generative modeling.

At a technical level, our analysis combines tools from measure transport with classical embedding results originating in the work of Whitney and Takens. These embedding theorems allow us to translate identifiability questions for maps and vector fields into topological conditions on finite collections of probability densities, yielding explicit bounds on the number of densities required for unique recovery.

We now formalize the setting. Let $M$ and $N$ be smooth, compact, $d$-dimensional Riemannian manifolds, and let $f \in C^1(M,N)$ be a diffeomorphism. For a probability measure $\rho \in \mathcal{P}(M)$, the pushforward $f_\# \rho \in \mathcal{P}(N)$ describes the redistribution of mass induced by $f$ (see Section~\ref{subsec:push} for a precise definition). In general, the pushforward of a single measure does not uniquely identify the underlying map, i.e., the equality $f_\# \rho = g_\# \rho$ does not imply $f = g$. However, in many applications one observes the action of $f$ on multiple measures $\rho_1, \ldots, \rho_m \in \mathcal{P}(M)$, which motivates the following question.
\begin{enumerate}
    \item[(Q1)] \label{q1}
    \textit{When do the $m$ pairs of measures
    \[
    \{(\rho_1, f_\# \rho_1), \ldots, (\rho_m, f_\# \rho_m)\}
    \]
    uniquely determine the map $f$?}
\end{enumerate}

Using the Whitney embedding theorem, we provide a positive answer to \hyperref[q1]{(Q1)}. We show that when $m > 2d+1$ there is a generic subset $\boldsymbol{D}$ of  strictly positive $C^1$ densities $D_+^1(M,\mathbb{R}^m)$, such that for $(\rho_1, \ldots, \rho_m) \in \boldsymbol{D}$, the pushforward action of a diffeomorphism $f$ uniquely determines $f$. Here, uniqueness means that 
\[
f_\# \rho_j = g_\# \rho_j \quad \text{for } 1 \le j \le m \quad \implies \quad  f=g
\]
for diffeomorphisms $f, g \in C^1(M,N)$. The term ``generic'' is used in a precise topological sense: the set of densities for which the result holds is open and dense in an appropriate function space.

We next consider the corresponding infinitesimal problem. Suppose $f = f_t$ is the time-$t$ flow map generated by a vector field $v$. Under mild regularity assumptions, the associated curve of measures $\rho_t = \left(f_t\right) _\# \rho$ satisfies the continuity equation
\[
\partial_t \rho_t + \textup{div} ( \rho_t v) = 0.
\]
Consequently, the first-order perturbation of $\rho$ induced by $v$ is given by
\[
\left. \partial_t \rho_t \right|_{t=0} = - \textup{div} ( \rho v).
\]
In general, the equality $\textup{div} (\rho v) = \textup{div} (\rho w)$ does not imply $v = w$, since weighted divergence-free components may remain invisible. As in \hyperref[q1]{(Q1)}, however, many applications involve observing the action of a vector field on multiple densities. This leads to the following question.
\begin{enumerate}
    \item[(Q2)] \label{q2}
    \textit{When do the $m$ density--divergence pairs
    \[
    \{(\rho_1, \textup{div} (\rho_1 v)), \ldots, (\rho_m, \textup{div} (\rho_m v))\}
    \]
    uniquely determine the vector field $v$?}
\end{enumerate}

Using similar topological arguments, we show that if $m > 2d + 1$, then for $(\rho_1,\dots,\rho_m)\in \boldsymbol{D},$ the generic set  constructed in response to \hyperref[q1]{(Q1)}, the equalities
\[
\textup{div} (\rho_j v ) = \textup{div} (\rho_j w ), \quad \text{for }1 \le j \le m, \quad \implies\quad  v = w.
\]
Thus, a finite collection of density-weighted divergence observations suffices to uniquely recover the underlying vector field. Our main results addressing \hyperref[q1]{(Q1)} and \hyperref[q2]{(Q2)} are summarized in Theorem~\ref{thm1}.

A more realistic data regime arises when the densities $\rho_j$ are not freely chosen, but are instead generated by an underlying time-dependent process. In many applications, measure-valued observations are collected sequentially in time, and successive densities are linked by the evolution of an unknown or partially known dynamical system. The simplest and most structured instance of this setting occurs when the densities are generated by the repeated pushforward action of a diffeomorphism, modeling either a discrete-time dynamical system or the stroboscopic sampling of a continuous-time physical process.

Concretely, suppose that there exists a diffeomorphism $h \in \mathrm{Diff}^2(M,M)$ such that
\[
\rho_j = \left( h^{j-1} \right)_\# \rho_1, \qquad j = 1,2,\ldots,m,
\]
where $\rho_1$ is an initial density and $h^{j-1}$ denotes the $(j-1)$-fold composition of $h$ with itself. In this time-dependent setting, the available measure-valued data are no longer independent inputs, but are instead dynamically correlated through the unknown map $h$. This raises a fundamental question: to what extent do the identifiability guarantees established in the static setting still hold when the data are generated along a single trajectory in density space? This motivates the following question.

 \begin{enumerate}
    \item[(Q3)] \label{q3}
    \textit{How do our answers to \textnormal{\hyperref[q1]{(Q1)}} and \textnormal{\hyperref[q2]{(Q2)}} change when the measure-valued data are time-dependent, that is, when
    $
    \rho_j = \left( h^{j-1} \right)_\# \rho_1
    $
    for some dynamical system $h \in \mathrm{Diff}^2(M,M)$?}
\end{enumerate}

While our analysis of \hyperref[q1]{(Q1)} and \hyperref[q2]{(Q2)} relied on Whitney's embedding theorem, the time-dependent problem \hyperref[q3]{(Q3)} requires a fundamentally different set of tools. Here we instead draw on Takens' time-delay embedding theory. We show that a carefully constructed delay-coordinate map, built from suitable quotients of pushforward densities, allows one to lift the identifiability results from the static setting to the time-dependent regime. Under certain assumptions on the dynamics (see Assumption~\ref{assumption:1}), this approach yields conditions under which \hyperref[q3]{(Q3)} admits a positive answer. Our main theoretical results for the time-dependent setting are summarized in Theorem~\ref{thm2}.

Takens' embedding theorem has played a central role in nonlinear time-series analysis and has inspired a wide range of data-driven methods for reconstruction, prediction, and control from partial observations \cite{takens2006detecting,noakes1991takens,pecora2007unified,raut2025arousal,sugihara1990nonlinear}. To our knowledge, this work is the first to make a direct connection between Takens-style delay embeddings and the well-posedness of measure-valued inverse problems involving dynamical system recovery from finitely many density snapshots.

Building on these identifiability results, we derive new metrics on the space of diffeomorphisms and smooth vector fields that are defined through the comparison of finitely many pushforward measures, or finitely many infinitesimal measure perturbations (see Corollary~\ref{cor1}). As a representative example, we show that for $m > 2d + 1$ strictly positive densities $(\rho_1, \ldots, \rho_m)$ belonging to the generic set $\boldsymbol{D}$, the quantity
\begin{equation}\label{eq:metric1}
\mathfrak{D}(f,g) := \sum_{j=1}^{m} \mathcal{D}\big(f _\# \rho_j,\, g _\# \rho_j\big),
\end{equation}
defines a genuine metric on the space of diffeomorphisms $\mathrm{Diff}^1(M,N)$. Here, $\mathcal{D}$ denotes any metric on the space of probability measures $\mathcal{P}(N)$, such as the Wasserstein distance \cite{villani2008optimal} or the Maximum Mean Discrepancy (MMD) \cite{gretton2012kernel}. Analogous constructions yield metrics for smooth vector fields via their action on densities through weighted divergence operators. These metrics are intrinsically adapted to measure-valued data and provide a general framework for solving inverse problems and training generative models using only finitely many observed distributions.

We further interpret our answers to \hyperref[q1]{(Q1)}--\hyperref[q3]{(Q3)} through the lens of data-driven dynamical systems, identifying conditions under which Perron--Frobenius and Koopman operators can be uniquely recovered from their action on finitely many densities or observables (see Section~\ref{subsec:dynamics}). In addition, we apply our main results to establish new well-posedness guarantees for inverse problems associated with continuity, advection, Fokker--Planck, and advection-diffusion-reaction equations (see Section~\ref{subsec:PDE}). Finally, we discuss the significance of our results in the context of generative models, highlighting how they can inform well-posedness analysis and the design of new architectures (see Section \ref{subsec:generative}). A schematic overview of the main results and their relationships is provided in Figure~\ref{fig:flowchart}.

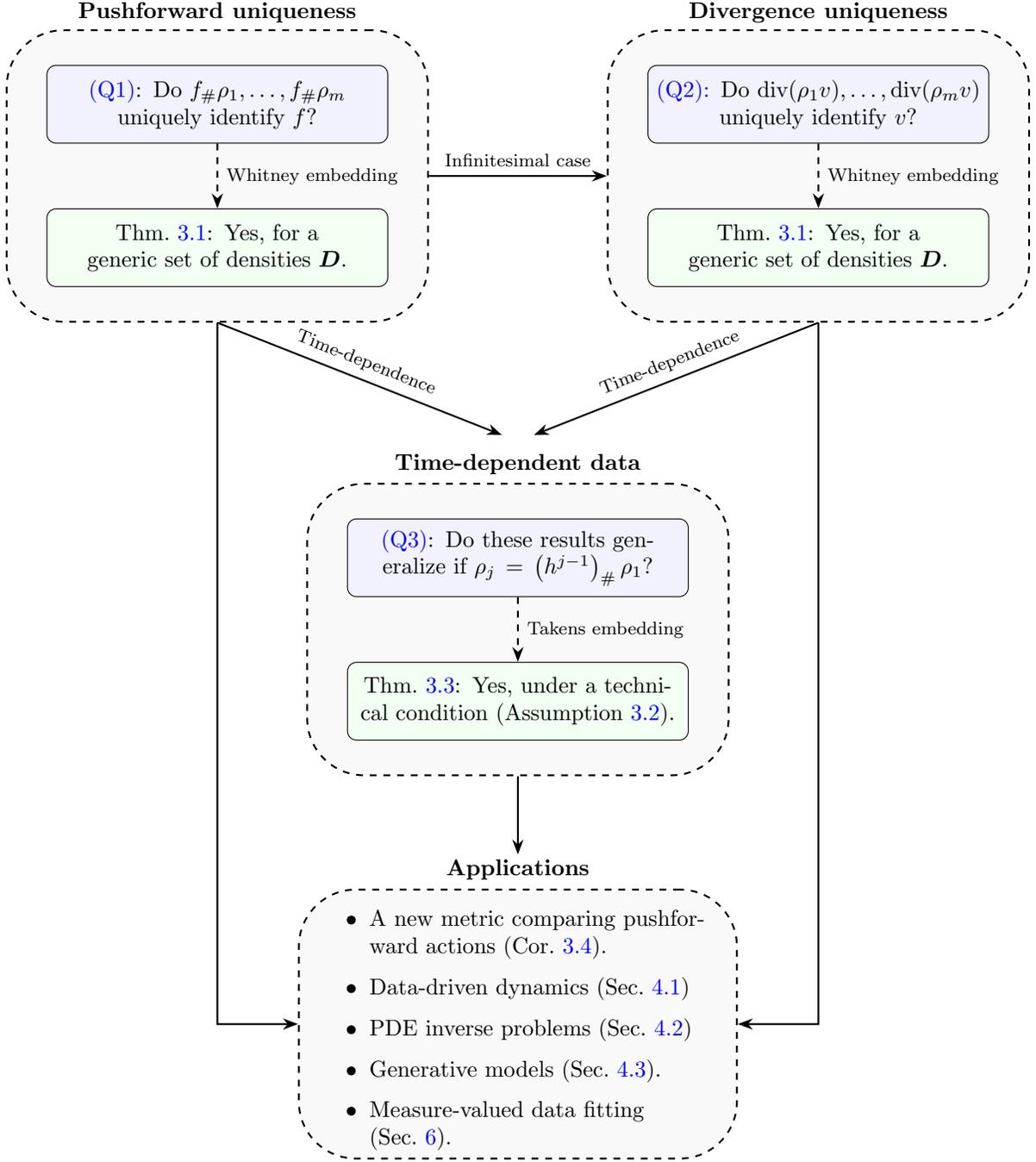
\begin{figure*}[h!]
\centering
\pgfdeclarelayer{background}
\pgfsetlayers{background,main}

\begin{tikzpicture}[scale = 0.7,
  box/.style={draw, minimum width=4cm, minimum height=1.2cm, text centered, rounded corners, text width=5cm, align=center},
  subbox/.style={draw, minimum width=3.5cm, minimum height=1cm, text centered, rounded corners, text width=3.5cm, align=center},
  arrow/.style={-Stealth, thick},
  dottedarrow/.style={-Stealth, thick, dotted},
  every node/.append style={font=\small},
  node distance=1.5cm
]

\node[box, fill = blue!5] (B1) {\hyperref[q1]{(Q1)}: Do $f_\#\rho_1,\ldots, f_\#\rho_m$ uniquely identify $f$?};
\node[box, below = 1cm of B1, fill = green!5] (B1b) {Thm.~\ref{thm1}:  Yes, for a generic set of densities $\boldsymbol{D}$.};

\node[box, right=4cm of B1, fill =  blue!5] (B2) {\hyperref[q1]{(Q2):} Do $\textup{div} (\rho_1 v ),\ldots, \textup{div} ( \rho_m v)$ uniquely identify $v$?};
\node[box, below = 1cm of B2, fill = green!5] (B2b) {Thm.~\ref{thm1}: Yes, for a generic set of densities $\boldsymbol{D}$.};

\node[box, fill =  blue!5] (B3)  at ($ (B1)!0.5!(B2) + (0,-10cm) $)  {\hyperref[q3]{(Q3)}: Do these results generalize if $\rho_j = \left( h^{j-1}\right)_\#\rho_1$?};
\node[box, below = 1cm of B3, fill = green!5] (B3b) {Thm.~\ref{thm2}: Yes, under a technical condition  (Assumption \ref{assumption:1}).};

\node[box, below=4.5cm of B3,
      thick,
       dash pattern=on 3pt off 3pt,
       text width = 6.5cm,
      minimum width=6.5cm,
      minimum height=3cm,
      rounded corners=25pt,
            fill=gray!5 ] (B4) {\vspace{-0.2cm}\begin{itemize}
    \item A new metric comparing pushforward actions (Cor.~\ref{cor1}).
    \item  Data-driven dynamics (Sec.~\ref{subsec:dynamics})
    \item PDE inverse problems (Sec.~\ref{subsec:PDE})
    \item Generative models (Sec.~\ref{subsec:generative}).
    \item  Measure-valued data fitting  (Sec.~\ref{sec:numerics}).
\end{itemize}};

\begin{pgfonlayer}{background}

\path let \p1 = ($(B1)!0.5!(B1b)$) in
  node[
      draw,
      thick,
      dash pattern=on 3pt off 3pt,
      minimum width=6.5cm,
      minimum height= 4.5cm,
      rounded corners=25pt,
      fill=gray!5
  ] (G1) at (\p1) {};

\path let \p2 = ($(B2)!0.5!(B2b)$) in
  node[
      draw,
      thick,
       dash pattern=on 3pt off 3pt,
      minimum width=6.5cm,
      minimum height=4.5cm,
      rounded corners=25pt,
            fill=gray!5 
  ] (G2) at (\p2) {};

\path let \p3 = ($(B3)!0.5!(B3b)$) in
  node[
      draw,
      thick,
       dash pattern=on 3pt off 3pt,
      minimum width=6.5cm,
      minimum height=4.5cm,
      rounded corners=25pt,
            fill=gray!5
  ] (G3) at (\p3) {};

\draw[arrow,dashed] (B1) -- (B1b) node[midway, right] {\scriptsize Whitney embedding};
\draw[arrow,dashed] (B2) -- (B2b) node[midway, right] {\scriptsize Whitney embedding};
\draw[arrow,dashed] (B3) -- (B3b) node[midway, right] {\scriptsize Takens embedding};

\node[above=-1pt of G1] {\textbf{Pushforward uniqueness}};
\node[above=-1pt of G2] {\textbf{Divergence uniqueness}};
\node[above=0pt of G3,fill=white] (L3) {\textbf{Time-dependent data}};
\node[above=0 of B4,fill=white] (L4) {\textbf{Applications}};

\draw[arrow] (G1.south) -- ([xshift = -10pt,yshift = 30pt] G3.north) node[midway, sloped, above] {\scriptsize Time-dependence};
\draw[arrow] (G1) -- (G2) node[midway, above] {\scriptsize Infinitesimal case};
\draw[arrow] (G1.south) |- (B4.west);
\draw[arrow] (G2.south) |- (B4.east);

\draw[arrow, shorten >=15pt] (G3) -- (B4);
\draw[arrow] (G2.south) -- ([xshift = 10pt,yshift=30pt] G3.north) node[midway, sloped, above] {\scriptsize Time-dependence};
\draw[arrow];

\end{pgfonlayer}

\end{tikzpicture}

\caption{Flowchart of our main results.}
\label{fig:flowchart}
\end{figure*}

The rest of the paper is structured as follows. Section \ref{sec:prelim} reviews preliminary notation and background material, including the classical embedding theorems of Whitney and Takens. Section \ref{sec:results} contains the statements of our main results and discussion of their significance. In Section~\ref{subsec:discussion}, we highlight applications  of our results across data-driven dynamical systems, PDE inverse problems, and generative models. The complete proofs of our main theorems and their corollaries then appear in Section~\ref{subsec:proofs}. Finally, in Section \ref{sec:numerics}, we present illustrative numerical experiments demonstrating unique pushforward map and vector field recovery from finite measure-valued datasets. Conclusions follow in Section~\ref{sec:conclusion}.

\section{Preliminaries}\label{sec:prelim}
We begin by establishing necessary notation and definitions that will be used throughout the paper. Section \ref{subsec:func}  introduces the relevant function spaces that play a role in our analysis, Section~\ref{subsec:push} defines the pushforward measure and change of variables formula, while Section~\ref{subsec:embed} introduces classical embedding theorems. 

\begin{table}[h]
    \centering
    \renewcommand{\arraystretch}{1.15}
        \rowcolors{2}{white}{gray!10}
    \begin{tabular}{|l|l|}
    \hline
    \textbf{Symbol} & \textbf{Meaning} \\ 
    \hline
        $M,N$ & Smooth, compact $d$-dimensional Riemannian manifolds   \\ 
        $\mathcal{P}(M)$ & Borel probability measures over $M$ \\
        $\langle\cdot,\cdot\rangle_{r_x}$ & Riemannian metric on $M$ at $x$ \\
        $C^{\ell}(M,\mathbb{R}^n)$ & $\ell$-times continuously differentiable functions \\
          $C_+^{\ell}(M,\mathbb{R}^n)$ & componentwise positive functions in $C^{\ell}(M,\mathbb{R}^n)$\\
          ${D}_+^{\ell}(M,\mathbb{R}^n)$ & functions in $C_+^{\ell}(M,\mathbb{R}^n)$ componentwise integrating to 1 \\ 
          $\textup{Diff}^\ell(M,N)$ & $C^\ell$ diffeomorphisms between $M$ and $N$ \\
        $\mathfrak{X}^\ell(M)$ & $C^\ell$ vector fields on $M$ \\
        $df_x$ & Differential of a function $f$ at the point $x$ \\
        $J_f(x)$ & Jacobian determinant of $f$ at $x$ \\
        $f_\#\rho$ & Pushforward of the measure $\rho$ under $f$\\
        $\Psi_{(y,f)}^{(k)}$ & $k$-dimensional delay map for observable $y$ and system $f$\\
                   $\boldsymbol{W}_k$ & Generic set of embeddings in $C^1(M,\mathbb{R}^{k})$, $k \geq 2d+1$ \\
        $\boldsymbol{G}$ & Generic set of pairs $(y,f)$ for which $\Psi_{(y,f)}^{(2d+1)}$ is an embedding \\
        \hline

    \end{tabular}
    \caption{Notation used throughout Section \ref{sec:results}.}
    
    \label{tab:placeholder}
\end{table}

\subsection{Defining the Function Spaces}\label{subsec:func}

Recall from Section \ref{sec:intro} that $M$ and $N$ denote smooth ($C^{\infty})$ compact $d$-dimensional Riemannian manifolds. Let $r$ be a smooth Riemannian metric on $M$ inducing the inner product $\langle \cdot, \cdot\rangle_{r_x}:T_x M \times T_x M \to \mathbb{R}$. Given $n\in \mathbb{N}$, we will write $C^1(M,\mathbb{R}^n)$ to denote the space of continuously differentiable maps between $M$ and $\mathbb{R}^n$, which is a Banach space when equipped with the norm 
\begin{equation}\label{eq:normC1}
    \| Y\|_{C^1}:= \sup_{x\in M}| Y(x)|+ \sup_{x\in M} \| dY_x\|.
\end{equation}
We write $|Y(x)|$ to denote the Euclidean norm of $Y(x)\in \mathbb{R}^n$ and $\|dY_x\|$ to denote the operator norm of the differential $dY_x:T_x M \to \mathbb{R}^n$, where $T_x M$ denotes the tangent space at $x\in M.$ 
Note that \eqref{eq:normC1} depends on the choice of Riemannian metric $r$ on $M$, as 
$$\| d Y_x\| := \sup_{v\in T_x M ,\, \|v\|_r = 1} | d Y_x(v)|,\qquad \|v\|_r = \sqrt{\langle v, v\rangle_{r_x}}.$$
Since $M$ is compact, any choice of the Riemannian metric will induce an equivalent topology on $C^1(M,\mathbb{R}^n)$.

More generally, $C^\ell(M,\mathbb{R}^n)$ is the space of $\ell$-times continuously differentiable $\mathbb{R}^n$-valued functions. Throughout, we write $Y_j(x)\in \mathbb{R}$ as shorthand for the $j$-th component of $Y(x)\in \mathbb{R}^n$, i.e., $Y_j(x):= Y(x)\cdot \mathbf{e}_j$, where $\mathbf{e}_j$ is the $j$-th standard unit basis vector in $\mathbb{R}^n$. Moreover, we define \begin{equation}\label{eq:posC1}
    C_+^\ell(M,\mathbb{R}^n):= \{Y\in C^\ell(M,\mathbb{R}^n): Y_j(x) > 0, \,\forall x\in M,\, 1\leq j \leq n\},
\end{equation}
which is the subset of $C^\ell(M,\mathbb{R}^n)$ functions that have strictly positive coordinate evaluations. Since we are primarily interested in working with probability densities, we also define 
\begin{equation}\label{eq:posD1}
    D_+^\ell(M,\mathbb{R}^n):= \Bigg\{ Y \in C_+^\ell(M,\mathbb{R}^n): \int_M Y_j(x) \, \textrm{d}x = 1,\, \,1\leq j \leq n\Bigg\},
\end{equation}
which is the subset of $C_+^\ell(M,\mathbb{R}^n)$ functions where each coordinate evaluation integrates to~1. We note that the integral in \eqref{eq:posD1} is computed with respect to the Riemannian volume, which depends on the underlying metric. Throughout, $C_+^\ell(M,\mathbb{R}^n)$ and $D_+^\ell(M,\mathbb{R}^n)$ are endowed with the subpsace topology inherited from $C^\ell(M,\mathbb{R}^n)$. We also write $\textup{Diff}^\ell(M,N)$ to denote the space of $C^\ell$ diffeomorphisms between $M$ and $N$, as well as $\mathfrak{X}^\ell(M)$ to denote the $C^\ell$ vector fields mapping $M\to TM$.

Given a function $y\in C^{1}(M,\mathbb{R})$ and a vector field $v\in \mathfrak{X}^1(M)$, we write $\nabla y$ and $\textup{div}(v)$ to denote the Riemannian gradient and divergence, respectively. Finally, given vector fields $v,w:M\to TM$ we write $\langle v,w\rangle$ as shorthand for the function $x\mapsto \langle v(x), w(x)\rangle_{r_x} $.

\subsection{Pushforward Measures}\label{subsec:push}
We will denote by $\mathcal{P}(M)$ the space of Borel probability measures over $M$. 
Given 
$\rho \in \mathcal{P}(M)$, its \textit{pushforward} under a measurable map $f:M\to N$ is the measure $f_\#\rho \in \mathcal{P}(N)$ defined by $(f_\#\rho)(B):= \rho(f^{-1}(B))$, for all Borel measurable sets $B\subseteq N$. As a slight abuse of notation we will also use $\rho$ to denote the corresponding density of the measure, when it exists. If $\rho \in \mathcal{P}(M)$ admits a $C^1$ density and $f\in \textup{Diff}^1(M,N)$, then the pushforward measure $f_\#\rho$ satisfies
\begin{equation}\label{eq:cov}
    (f_\#\rho)(x)= \rho\left(f^{-1}(x) \right)|\det df_x^{-1}|,\qquad \forall x\in N.
\end{equation}
The formula \eqref{eq:cov} is known as a \textit{change of variables} formula. In \eqref{eq:cov}, $df_x^{-1}:T_x N \to T_{f^{-1}(x)} M$ is the differential of the map $f^{-1}$ evaluated at $x\in N$, and the term $|\det df_x^{-1}|$ describes how volumes are distorted under the pushforward operation, ensuring that $f_\#\rho$ remains a probability measure.  Throughout, we also write $J_{f^{-1}}(x)$ to denote the Jacobian determinant of $f^{-1}$ at $x\in N$.  

\subsection{Embedding Theorems}\label{subsec:embed}
The classical embedding theorems due to Hassler Whitney and Floris Takens play a central role in our analysis. To this end, we first precisely define what is meant by an embedding.

\begin{definition}[Embedding]\label{def:embedding}
    A map $Y\in C^1(M,\mathbb{R}^n)$ is said to be an \textit{embedding} if 
    \begin{enumerate}
        \item[\textit{(i)}] $Y$ is injective;
        \item[\textit{(ii)}] $dY_x:T_xM \to \mathbb{R}^n$ is injective, for all $x\in M$;
        \item[\textit{(iii)}] $Y$ is a homeomorphism onto its image.
    \end{enumerate}
\end{definition}
In the subsequence discussions, we are primarily concerned with embeddings of compact manifolds, in which case Definition \ref{def:embedding}\textit{(iii)} does not need to be checked. In this setting, \textit{(iii)} follows from \textit{(i)}, as $Y$ can be viewed as an invertible continuous map between compact sets, which can be shown to be a homeomorphism; see \cite[Prop.~13.26]{sutherland2009introduction}.

The Whitney embedding theorem asserts that the collection of embeddings in $C^1(M,\mathbb{R}^k)$ is ``topologically large" if $k \geq 2d+1$. In particular, Whitney showed that the set of embeddings is \textit{generic}, i.e., it is open and dense in the $C^1(M,\mathbb{R}^k)$ topology. 
\begin{definition}[Generic]
    Let $X$ be a topological space. A subset $S\subseteq X$ is said to be \textit{generic} if $S$ is open and dense. 
\end{definition}

\begin{theorem}[Whitney Embedding \cite{hirsch2012differential}]\label{thm:whitney}
   Let $k \geq 2d+1$. Then, the set 
       $$\boldsymbol{W}_k:=\{Y\in C^1(M,\mathbb{R}^k): Y\textup{ is an embedding}\}$$
 is generic in $C^1(M,\mathbb{R}^k)$.
\end{theorem}
At an intuitive level, the genericity assumption in Theorem~\ref{thm:whitney} reflects two complementary facts: first, that embeddings are stable under small $C^1$ perturbations, and second, that within the space $C^1(M,\mathbb{R}^k)$, embeddings form a dense subset. In other words, any sufficiently small perturbation of an embedding remains an embedding, and any $C^1$ map can be approximated arbitrarily well by one.

Takens' embedding theorem extends this perspective to a dynamical setting, where the collection of observables used to form the embedding is no longer freely chosen, but instead arises from successive partial observations along the time evolution of a dynamical system. In this sense, Takens' theorem can be viewed as a dynamical analogue of Whitney's result, replacing independent observables with time-delayed measurements while retaining generic embedding guarantees. We now define the time-delay map and state Takens' embedding theorem.
\begin{definition}[Time-Delay Map]\label{def:delay_map}
    Let $y\in C^1(M,\mathbb{R})$, let $h\in \textup{Diff}^1(M,M)$, and let $k\in \mathbb{N}$. The \textit{time-delay map} $\Psi_{(y,h)}^{(k)}:M\to \mathbb{R}^k$ is defined by setting 
    $$\Psi_{(y,h)}^{(k)}(x):= (y(x),y(h(x)),\dots, y(h^{k-1}(x))),\qquad x\in M.$$
\end{definition}
\begin{theorem}[Takens Embedding \cite{noakes1991takens}]\label{thm:takens}
   There exists a generic subset $\boldsymbol{G}\subset C^1(M,\mathbb{R})\times \textup{Diff}^1(M,M)$ such that for all  $(y,h)\in \boldsymbol{G}$ it holds that $\Psi_{(y,h)}^{(2d+1)}$ is an embedding of $M$. 
   \end{theorem}

   The time-delay map can be constructed based only upon partial observations $\{y(h^j(x))\}_j$ of the trajectory $\{h^j(x)\}_j$ and gives rise to a new dynamical system in \textit{time-delay coordinates}, i.e., $\Psi_{(y,h)}^{(2d+1)}(x)\mapsto \Psi_{(y,h)}^{(2d+1)}(h(x))$, which is topologically equivalent to the dynamics $x\mapsto h(x)$ on $M$ and preserves crucial properties including the structure of periodic orbits and Lyapunov exponents. Since appending smooth components to an embedding preserves injectivity, Theorem \ref{thm:takens} also implies that $\Psi_{(y,h)}^{(k)}$ is an embedding for any $k \geq 2d+1$, provided that $(y,h)\in \boldsymbol{G}$.

\section{Main Results}\label{sec:results}
In this paper, we organize the discussion around the central questions \hyperref[q1]{(Q1)}-\hyperref[q3]{(Q3)}, which ask when finite collections of measure-valued data suffice for \textit{unique identification}. Specifically, given a \emph{finite} set of probability measures $\rho_1, \ldots, \rho_{m}$, we seek conditions under which the following implications hold:
\begin{enumerate}
    \item[(K1)]\label{i}
    For diffeomorphisms $f,g \in \mathrm{Diff}^1(M,N)$,
    \[
    f_\# \rho_j = g_\# \rho_j \quad \text{for } 1 \le j \le m
    \]
    implies that $f = g$.
    
    \item[(K2)]\label{ii}
    For vector fields $v,w \in \mathfrak{X}^1(M)$,
    \[
    \textup{div} (\rho_j v) = \textup{div} (\rho_j w) \quad \text{for } 1 \le j \le m
    \]
    implies that $v = w$.
\end{enumerate}

The conclusion \hyperref[i]{(K1)} is helpful for guaranteeing that one can uniquely identify a diffeomorphism from its pushforward action on $m$ densities. This is important for understanding the well-posedness of regression over the space of probability measures, which includes certain applications in data-driven dynamical systems (see Section \ref{subsec:dynamics}) and many generative modeling tasks (see Section \ref{subsec:generative}). The conclusion \hyperref[ii]{(K2)} is essential for analyzing the well-posedness of certain PDE inverse problems where one wishes to uniquely identify the differential operator from its action on finitely many densities. While in this section we focus on establishing criteria under which \hyperref[ii]{(K2)} holds, in Section \ref{subsec:PDE} we use these guarantees to establish new results for the identification of vector fields governing continuity, advection, and advection-diffusion-reaction equations.

\paragraph{Time-Independent  Recovery.} We now present our main results concerning the unique recovery of pushforward maps from finite measure-valued data.
Our first main result, Theorem~\ref{thm1}, shows that there is a generic set of densities in $D_+^1(M,\mathbb{R}^{m})$ such that the pushforward action on these densities uniquely identifies the diffeomorphism, while the derivative of these densities along a flow uniquely recovers the vector field.
\begin{theorem}\label{thm1}
    Fix $m > 2d + 1$. There exists a generic subset $\boldsymbol{D} \subset D_+^1(M,\mathbb{R}^{m})$ such that for every $(\rho_1,\ldots,\rho_{m}) \in \boldsymbol{D}$, both \textup{\hyperref[i]{(K1)}} and \textup{\hyperref[ii]{(K2)}} hold.
\end{theorem}
Note that the required number of densities in Theorem~\ref{thm1}  is one greater than the embedding dimension in Whitney's embedding theorem (see Theorem~\ref{thm:whitney}). This is a consequence of our proof technique, in which we divide the first $m-1$ pushforward densities, $f_\# \rho_1,\ldots,f_\# \rho_{m-1}$, by the $m$-th pushforward density $f_\# \rho_{m}$, thereby canceling all Jacobian determinant factors appearing in the change of variables formula~\eqref{eq:cov}. Following this cancellation, we then use tools from  Whitney's embedding theorem to obtain the desired uniqueness results. The full proof is presented in Section~\ref{subsec:proofs}.

\paragraph{Time-Dependent Recovery.}
While Theorem~\ref{thm1} establishes unique recovery for generic densities in $D_+^1(M,\mathbb{R}^{m})$, it does not address the case in which the densities $\rho_j$ are generated by a time-dependent process. In particular, situations where the measures arise through iterated pushforward by a dynamical system,
\begin{equation}\label{eq:rho_j}
\rho_j = (h^{j-1})_\# \rho_1, \qquad 1 \le j \le m,
\end{equation}
for some diffeomorphism $h$, fall outside the scope of Theorem~\ref{thm1}. In \eqref{eq:rho_j}, $h^0$ denotes the identity map.  Such temporally correlated data, however, are common in practical applications, where distributions are observed sequentially along the evolution of an underlying system. To treat this setting, we introduce the following technical assumption.
\begin{assumption}\label{assumption:1}
    We assume $(\rho_1,h)\in D_+^1(M)\times \textup{Diff}^2(M,M)$ satisfies $(\rho_1/h_\#\rho_1,h)\in \boldsymbol{G}$. Here, $\boldsymbol{G}$ is the generic set of observables and dynamical systems appearing in Theorem~\ref{thm:takens}.
\end{assumption}
Assumption \ref{assumption:1} implies that  $\Psi_{(y,h)}^{(k)}=(y(x),y(h(x)),\dots,y(h^{k-1}(x)))$ is an embedding where $y:=\rho_1/h_\#\rho_1$ and $k \geq 2d+1$. Note that Assumption~\ref{assumption:1} imposes an additional degree of smoothness on the dynamics, requiring $h \in \mathrm{Diff}^2(M,M)$. This condition ensures that the observable $\rho_1 / h_\# \rho_1$ lies in $C^1(M,\mathbb{R})$, which is necessary for the application of Takens-type embedding arguments. 
Under Assumption~\ref{assumption:1}, we are able to extend the identifiability results of Theorem~\ref{thm1} to the time-dependent setting.

\begin{theorem}\label{thm2}
Fix $m > 2d + 1$. Let $(\rho_1, h) \in D_+^1(M) \times \mathrm{Diff}^2(M,M)$ satisfy Assumption~\ref{assumption:1}, and define
$\rho_1,\ldots,\rho_{m}$ following~\eqref{eq:rho_j}. Then both \textup{\hyperref[i]{(K1)}} and \textup{\hyperref[ii]{(K2)}} hold.
\end{theorem}
 
The proof of Theorem \ref{thm2} mirrors that of Theorem~\ref{thm1}, canceling Jacobian factors by taking ratios of the corresponding pushforward densities. The key difference is that we invoke Takens' embedding theorem (Theorem~\ref{thm:takens}) rather than Whitney’s (Theorem~\ref{thm:whitney}). The proof is postponed to Section~\ref{subsec:proofs}.

\paragraph{Pushforward- and Divergence-Based Metrics.} As a corollary of Theorems~\ref{thm1} and~\ref{thm2}, we define new metrics over the space of diffeomorphisms and vector fields based on comparison of pushforward measures and divergence operators evaluated on \textit{finite} collections of densities.

\begin{corollary}[Metrics via pushforward and divergence operators]\label{cor1}
Let $m > 2d + 1$. Let $\mathcal{D}(\cdot,\cdot)$ be a metric on $\mathcal{P}(N)$ and let $\mathbf{d}(\cdot,\cdot)$ be a metric on $C(M,\mathbb{R})$. Suppose that either
\begin{enumerate}[label=(\alph*)]
    \item $(\rho_1,\ldots,\rho_{m}) \in \boldsymbol{D}$, where $\boldsymbol{D}$ is the generic set appearing in Theorem~\ref{thm1}, or
    \item $\rho_j := h^{j-1} _\# \rho_1$ for $1 \le j \le m$, where the pair $(\rho_1,h)$ satisfies Assumption~\ref{assumption:1}.
\end{enumerate}
Then the following hold:
\begin{enumerate}
    \item[(i)] The function
    \[
    \mathfrak{D}(f,g) := \sum_{j=1}^{m} \mathcal{D}\big(f _\# \rho_j,\, g _\# \rho_j\big)
    \]
    defines a metric on $\mathrm{Diff}^1(M,N)$.
    
    \item[(ii)] The function
    \[
    \mathsf{D}(v,w) := \sum_{j=1}^{m} \mathbf{d}\big(\textup{div} (\rho_j v),\, \textup{div} (\rho_j w)\big)
    \]
    defines a metric on $\mathfrak{X}^1(M)$.
    
\end{enumerate}
\end{corollary}

Additionally, it is straightforward to verify that if $\|\cdot\|$ is any norm on $C(M,\mathbb{R})$, then by linearity of the divergence operator
\begin{equation}\label{eq:norm}
    \| v \|_\mathsf{D}= \sum_{j=1}^{m}\left\| \textup{div}( \rho_j v)\right\|, \qquad v \in \mathfrak{X}^1(M),
\end{equation}
 defines a norm on $\mathfrak{X}^1(M)$ when assumption (a) or (b) of Corollary \ref{cor1} holds.

 In Section~\ref{subsec:generative}, we discuss stability properties of  the metric $\mathfrak{D}$ in the context of generative modeling, and in Section \ref{sec:numerics},  we use $\mathfrak{D}$ and $\mathsf{D}$ to construct loss functions which can be used to numerically recover pushforward maps and vector fields from finite measure-valued data.
 
\section{Application to Learning from Distributions}\label{subsec:discussion}

In this section, we examine the implications of the main theoretical results presented earlier in Section~\ref{sec:results} for a range of measure-valued learning tasks that arise throughout data-driven science and engineering. In Section~\ref{subsec:dynamics}, we consider applications to data-driven dynamical systems, focusing on the unique recovery of Perron--Frobenius and Koopman operators. In Section~\ref{subsec:PDE}, we establish new theoretical guarantees for the unique solution of certain PDE inverse problems, including those associated with continuity, Fokker--Planck, and advection-diffusion-reaction equations. Finally, in Section~\ref{subsec:generative}, we discuss how our results may inform the design and analysis of generative models.

\subsection{Unique Recovery of Perron--Frobenius and Koopman Operators}\label{subsec:dynamics}

\paragraph{Perron--Frobenius Operator Recovery.}
Across a wide range of physical, biological, and engineering applications, one seeks to infer an underlying dynamical evolution rule $f : M \to M$ from observed measurement data. In some settings, the available data consists of state trajectories $\{ f^j(x) \}_j$ generated from a fixed initial condition $x \in M$, in which case model identification proceeds by directly fitting the observed trajectories.

A second, and increasingly common, data regime arises when one observes a temporal sequence of probability distributions $\{ \rho(t_j) \}_j \subset \mathcal{P}(M)$, reflecting uncertainty, noise, or population-level variability in the system state. In this setting, the dynamics are naturally described by the Perron--Frobenius operator (PFO), also known as the transfer operator, which governs the evolution of densities under the action of the map $f$. The PFO lifts the finite-dimensional, nonlinear dynamics on $M$ to a linear evolution on an infinite-dimensional space of measures, and has been widely used for modeling, analysis, and prediction of dynamical systems from data across diverse biological and engineering applications \cite{schutte2001transfer,froyland2010coherent,klus2015numerical}.
\begin{definition}[Perron--Frobenius Operator \cite{lasota2013chaos}]
    Given a measure space $(X,\mathscr{B},\mu)$, and a non-singular dynamical system $f:X\to X$, i.e., $\mu(B) = 0 $ implies $\mu(f^{-1}(B)) = 0$ for all $B\in \mathscr{B}$, the PFO is the unique linear operator $\mathcal{T}:L_{\mu}^1(X) \to L_{\mu}^1(X)$ defined by the relationship 
    \begin{equation}\label{eq:PFO}
        \int_B \mathcal{T}\phi\, \textrm{d}\mu = \int_{f^{-1}(B)} \phi \, \textrm{d}
        \mu,\qquad \forall B\in \mathscr{B},\qquad \phi \in L_{\mu}^1(X).
    \end{equation}
\end{definition}

In our setting, we take $X = M$ to be a smooth, compact Riemannian manifold, let $\mathscr{B}$ denote the Borel $\sigma$-algebra on $M$, and let $\mu$ be the associated volume measure. For a diffeomorphism $f \in \mathrm{Diff}^1(M,M)$, the PFO defined in~\eqref{eq:PFO} reduces, when acting on densities, to the classical change-of-variables formula. In particular, for any $\rho \in D_+^1(M,\mathbb{R})$, we have
\[
(\mathcal{T}\rho)(x) = \rho\big(f^{-1}(x)\big)\, J_{f^{-1}}(x) = (f_\# \rho)(x),
\]
where $J_f$ denotes the Jacobian determinant of $f$.

As a consequence, Theorems~\ref{thm1} and~\ref{thm2} admit an immediate reformulation in terms of Perron--Frobenius operators. From this perspective, Theorem~\ref{thm1} asserts that the PFO associated with a diffeomorphism is \emph{uniquely determined} by its action on $m$ smooth densities belonging to a generic set. More precisely, let $\mathcal{T}_f$ and $\mathcal{T}_g$ denote the PFOs corresponding to $f,g \in \mathrm{Diff}^1(M,M)$. If $(\rho_1,\ldots,\rho_{m}) \in \boldsymbol{D}$, where $\boldsymbol{D}$ is the generic set appearing in Theorem~\ref{thm1}$,$ and
\[
\mathcal{T}_f \rho_j = \mathcal{T}_g \rho_j, \qquad 1 \le j \le m,
\]
then it follows that $f = g$, and hence $\mathcal{T}_f = \mathcal{T}_g$.

Similarly, Theorem~\ref{thm2} provides conditions under which a finite trajectory of densities
\[
\{\mathcal{T}^j \rho : 1 \le j \le m\}
\]
generated by repeated application of the PFO suffices to recover the underlying dynamical operator. This interpretation highlights the relevance of our results to data-driven identification of transfer operators from finitely many measure-valued observations.

 \paragraph{Koopman Operator Recovery.}
 The PFO is adjoint to the well-known Koopman operator, which has also been used in a variety of data-driven prediction, estimation, and control applications \cite{mezic2013analysis,klus2015numerical,colbrook2025introductory}. While the PFO acts on the space of densities, the Koopman operator evolves observables $y:X\to \mathbb{R}$, which are measurement functions mapping the state of the dynamical system to a real number. 
\begin{definition}[Koopman Operator \cite{lasota2013chaos}]
    Given a measure space $(X,\mathscr{B},\mu)$ and a non-singular system $f:X\to X$, the Koopman operator $\mathcal{K}:L_{\mu}^{\infty}(X)\to L_{\mu}^{\infty}(X)$ is defined by $\mathcal{K} \phi = \phi \circ f$.
\end{definition}
 
While the main theory of this paper focuses on the unique recovery of Perron--Frobenius operators from measure-valued data, analogous questions can be posed for the Koopman operator. In particular, one may ask whether a dynamical system or vector field can be uniquely identified from its action on a finite collection of scalar observables.

In direct analogy with \textup{\hyperref[i]{(K1)}} and \textup{\hyperref[ii]{(K2)}}, we consider the following conditions. Given a finite set of observables $y_1,\ldots,y_m \in C^1(M,\mathbb{R})$, we ask when the implications below hold:
\begin{enumerate}
    \item[(K3)]\label{iii}
    For diffeomorphisms $f,g \in \mathrm{Diff}^1(M,M)$,
    \[
    y_j \circ f = y_j \circ g \quad \text{for } 1 \le j \le m
    \]
    implies that $f = g$.
    
    \item[(K4)]\label{iv}
    For vector fields $v,w \in \mathfrak{X}^1(M)$,
    \[
\langle \nabla y_j, v\rangle = \langle \nabla y_j , w\rangle\quad \text{for } 1 \le j \le m
    \]
    implies that $v = w$.
\end{enumerate}

We now restate Theorem~\ref{thm1} in the context of Koopman operators.

\begin{proposition}\label{prop:koopman1}
Fix $m \ge 2d + 1$. There exists a generic subset $\boldsymbol{W}_m \subset C^1(M,\mathbb{R}^m)$ such that, for every $(y_1,\ldots,y_m) \in \boldsymbol{W}_m$, both \textnormal{\hyperref[iii]{(K3)}} and \textnormal{\hyperref[iv]{(K4)}} hold.
\end{proposition}

Similarly, there is an analogue of Theorem~\ref{thm2} for Koopman operators.

\begin{proposition}\label{prop:koopman2}
Fix $m \ge 2d + 1$. There exists a generic subset $\boldsymbol{G} \subset C^1(M,\mathbb{R}) \times \mathrm{Diff}^1(M,M)$ such that for every pair $(y,h) \in \boldsymbol{G}$, defining the observables
\[
y_j := y \circ h^{j-1}, \qquad 1 \le j \le m,
\]
both \textnormal{\hyperref[iii]{(K3)}} and \textnormal{\hyperref[iv]{(K4)}} hold.
\end{proposition}

The proofs of Propositions~\ref{prop:koopman1} and~\ref{prop:koopman2}, which are presented in Section~\ref{subsec:proofs}, are considerably less involved than those of Theorems~\ref{thm1} and~\ref{thm2}. This simplification stems from the fact that, in the Koopman setting, no Jacobian determinant arises from the action of the operator. In addition, Proposition~\ref{prop:koopman2} does not require any auxiliary technical assumptions, and the number of observables needed for identifiability is $m \geq 2d+1$ rather than $m> 2d+1$.

Both propositions follow as direct applications of Whitney’s and Takens’ embedding theorems. Finally, in analogy with Corollary~\ref{cor1}, one may also define metrics on the spaces of diffeomorphisms and vector fields by comparing the action of Koopman operators on the finite family of observables $\{y_j\}$. 

\begin{corollary}[Metrics via composition and gradient operators]\label{cor2}
Let $m \ge 2d + 1$ and let $\mathbf{d}(\cdot,\cdot)$ be a metric on $C(M,\mathbb{R})$. Suppose that either
\begin{enumerate}[label=(\alph*)]
    \item $(y_1,\ldots,y_{m}) \in \boldsymbol{W}_m$, where $\boldsymbol{W}_m$ is the generic set appearing in Proposition~\ref{prop:koopman1}, or
    \item $y_j:= y_1\circ h^{j-1}$ for $1 \le j \le m$ where  $(y,h)\in \boldsymbol{G}$, the generic set from Proposition~\ref{prop:koopman2}.
\end{enumerate}
Then the following hold: 
\begin{enumerate}
    \item[(i)] The function
    \[
    \mathfrak{D}(f,g) := \sum_{j=1}^{m} \mathbf{d}\big(y_j\circ f,\, y_j\circ g\big)
    \]
    defines a metric on $\mathrm{Diff}^1(M,M)$.
    
    \item[(ii)] The function
    \[
    \mathsf{D}(v,w) := \sum_{j=1}^{m} \mathbf{d}\big(\langle \nabla y_j , v\rangle,\, \langle \nabla y_j, w\rangle\big)
    \]
    defines a metric on $\mathfrak{X}^1(M)$.  
\end{enumerate}
\end{corollary}
Similar to \eqref{eq:norm}, if $\|\cdot\|$ is a norm over $C(M,\mathbb{R})$ then

$$
\|v\|_{\mathsf{D}} = \sum_{j=1}^m \| \langle \nabla \rho_j ,v \rangle \|, \qquad v \in \mathfrak{X}^1(M),
$$
defines a norm on $\mathfrak{X}^1(M)$ whenever assumptions (a) or (b) of Corollary \ref{cor2} hold.
\subsection{PDE Inverse Problems}\label{subsec:PDE}
Our theory can be directly applied to provide new well-posedness guarantees for inverse problems arising from evolution equations of the form 
\begin{equation}\label{eq:PDE_1}
    \partial_t \rho(t,x) + \mathcal{L}[\rho(t,x)] = 0, \qquad \rho(0,x) = \rho_0(x), \qquad t\in (0,T),\qquad x\in M.
\end{equation}
In particular, we consider \eqref{eq:PDE_1} based upon the following choices for $\mathcal{L}[\cdot]$:
\begin{equation}\label{eqPDE}
    \begin{cases}
        \text{Continuity Eqn. (CE)} & \mathcal{L}[\rho] = \textup{div} (\rho v)   \\
        \text{Advection Eqn. (AE)}  &  \mathcal{L}[\rho] = \langle\nabla \rho , v\rangle  \\ 
        \text{Advection-Diffusion-Reaction (ADR) Eqn.} & \mathcal{L}[\rho] =  \textup{div} (\rho v) - \textup{div} (D \nabla \rho) + R(\rho)  
    \end{cases}.
\end{equation}

In \eqref{eqPDE}, $v:M\to TM$ is a vector field, $D:TM\to TM$ is a symmetric positive-definite diffusion tensor, and $R(\cdot)$ is a reaction functional. Collectively, these equations model a wide range of complex physical and biological phenomena, from fluid transport to chemically reacting systems. They have also been employed in numerous computational inverse problems, including the identification of tumor-growth dynamics from patient-specific data \cite{hogea2008image,liu2014patient,tong2020trajectorynet,liu2021solving,blickhan2025dice}. 

Note that when $R(\rho) = 0$ the ADR equation reduces to a Fokker--Planck equation, and when additionally $D = 0$ it reduces to the CE. While the theoretical foundations of many inverse problems associated with \eqref{eqPDE} have been studied extensively, rigorous guarantees for the unique recovery of the underlying vector field $v$ from \emph{finite snapshot data} remain limited. In this section, we leverage our main theoretical results to address the following question:

\begin{enumerate}
    \item[(Q4)] \label{Q4}
    \textit{Given finitely many solution snapshots $\{\rho(t_j,x)\}_{j=0}^{m}$ of Equation~\eqref{eq:PDE_1}, and assuming that the diffusion tensor $D$ and reaction functional $R$ are known, under what conditions does this data uniquely determine the vector field $v$?}
\end{enumerate}

Throughout this section, we assume that the solution snapshots are uniformly spaced in time, i.e., $t_j = j\,\Delta t \in [0,T]$. 
By interpreting the solution operators of the CE and AE as pushforward and composition operators, respectively, we are able to use our main time-dependent results (Theorem \ref{thm2} and Proposition \ref{prop:koopman2}) to uniquely recover the time-$\Delta t$ flow map $f_{\Delta t}:M\to M$ from the observed data $\{\rho(t_j,x)\}_{j=0}^{m}$.  In order to  recover the vector field $v$, we must assume access to additional data incorporating time-derivatives of the observed states: 
\begin{equation*}\label{eq:data_change}
    \{ (\rho(t_j,x),\, \partial_t \rho(t_j,x)) \}_{j=0}^{m} , \qquad \text{ or equivalently } \qquad \{ (\rho(t_j,x),\, \mathcal{L}[\rho(t_j,x))] \}_{j=0}^{m} .
\end{equation*}
In what follows, we write $f_{\Delta t} = f^{(v)}_{\Delta t}$, $\mathcal{L} = \mathcal{L}^{(v)}$, and $\rho(t) = \rho^{(v)}(t)$ to highlight the dependence of the flow map, differential operator, and PDE solution on the vector field $v$.
    \begin{corollary}[Inversion for CEs and AEs]\label{cor:CE}
Let $v,w \in \mathfrak{X}^2(M)$, fix $m > 2d+1$, and let $t \mapsto \rho^{(v)}(t)$ and $t \mapsto \rho^{(w)}(t)$, $t \in [0,T]$, be strong solutions of either
\begin{enumerate}
    \item[(a)] continuity equations (CEs), where $(\rho_0, f^{(v)}_{\Delta t})$ satisfies Assumption~\ref{assumption:1}, or
    \item[(b)] advection equations (AEs), where $(\rho_0, f^{(v)}_{\Delta t}) \in \boldsymbol{G}$, the generic set from Theorem~\ref{thm:takens}.
\end{enumerate}
If $\rho^{(v)}(t_j,x) = \rho^{(w)}(t_j,x)$ for all $x \in M$ and $0 \leq j \leq m$, then the time-$\Delta t$ flow maps coincide pointwise, i.e., 
\[
f_{\Delta t}^{(v)}(x) = f_{\Delta t}^{(w)}(x) \quad \text{ for all }  x \in M.
\]
If, in addition, $
\mathcal{L}^{(v)}\big[\rho^{(v)}(t_j,\cdot)\big](x) = \mathcal{L}^{(w)}\big[\rho^{(w)}(t_j,\cdot)\big](x)$ for all $x \in M$ and $0 \leq j \leq m-1$, then the vector fields coincide pointwise, i.e.,
\[
v(x) = w(x) \quad \text{for all } x \in M.
\]
\end{corollary}

In Corollary \ref{cor:CE}, we assume sufficient regularity on the vector fields $v,w\in \mathfrak{X}^2(M)$ and initial data $\rho_0\in C^1(M,\mathbb{R})$ such that $t\mapsto \rho^{(v)}(t)$ and $t\mapsto \rho^{(w)}(t)$ yield classical solutions of the CE and AE. Further work is needed to extend Corollary \ref{cor:CE} to encompass ADR equations. Indeed, our main analytical technique for proving Corollary \ref{cor:CE} involves interpreting the solution map to CEs and AEs as pushforward and composition operators over the space $M$, while the solution map of the ADR equation generally lacks this structure. However, we can still provide guarantees for the unique recovery of the vector field governing an ADR equation when we observe the action of its differential operator on $m$ initial conditions belonging to the generic set $\boldsymbol{D}$. 
\begin{corollary}[Inversion for ADR equations]\label{cor:ADR}
Fix $m > 2d+1$, and let $(\rho_1,\dots,\rho_{m})\in \boldsymbol{D}$, where $\boldsymbol{D}$ is the generic set from Theorem~\ref{thm1}. Further assume that
     $v,w\in \mathfrak{X}^1(M)$, that  the diffusion tensor $D:TM\to TM$ is bounded and measurable, and that  the reaction functional $R:L^1(M)\to L^1(M)$. If $$\mathcal{L}^{(v)}[\rho_j] = \mathcal{L}^{(w)}[\rho_j], \quad \text{ weakly for } 1\leq j \leq m,$$  where $\mathcal{L}$ is the differential operator of the ADR equation, then 
     \[
v(x) = w(x) \quad \text{for all } x \in M.
\] 
\end{corollary}

In Corollary \ref{cor:ADR}, we have assumed only boundedness and integrability of the diffusion tensor and reaction functional, respectively. Moreover, $\boldsymbol{D}\subseteq C^1(M,\mathbb{R}^{m})$ while the ADR operator $\mathcal{L}[\cdot]$ involves first- and second-order spatial derivatives.  Thus, to make sense of the quantity $\mathcal{L}[\rho_j]$ in Corollary \ref{cor:ADR}, these derivatives are interpreted in the weak sense by integration against test functions in $C_c^{\infty}(M)$. These details, along with the complete proofs of Corollary~\ref{cor:CE} and Corollary~\ref{cor:ADR}, appear in Section~\ref{subsec:proofs}.

\subsection{Generative models}\label{subsec:generative}

Many modern generative models are formulated in terms of measure transport between a known reference distribution and a target data distribution. Prominent examples include diffusion models~\cite{ho2020denoising}, normalizing flows~\cite{papamakarios2021normalizing}, and flow matching methods~\cite{lipman2022flow}. Normalizing flows aim to learn an explicit pushforward map between the two distributions, either through a discrete transformation or as the flow of a time-dependent vector field, whereas diffusion models connect the source and target distributions via a continuous solution path governed by a Fokker--Planck equation. Depending on the training objective and modeling choices, some of these approaches admit a unique transport map or flow vector field, while others are inherently underconstrained and allow multiple minimizers.

Uniqueness guarantees in generative modeling are therefore of fundamental importance, as they form a prerequisite for more refined notions of well-posedness, including stability. In particular, stability analysis is essential for understanding the robustness of learned generative models to data noise, model misspecification, and adversarial perturbations. From this perspective, our results provide a foundational characterization of when inverse problems based on measure transport admit unique solutions. This, in turn, creates a pathway toward systematic stability analyses of existing generative models and may inform the design of new generative modeling frameworks with built-in identifiability guarantees.

More concretely, consider a probability measure $\nu\in \mathcal{P}(M)$ representing the target data distribution. The goal of generative modeling is to find some $f$ which satisfies $f_\#\rho = \nu$ for a reference noise distribution  $\rho\in \mathcal{P}(M)$ that is easy to sample from, e.g., a multivariate Gaussian measure. In general situations, there are infinitely many such $f$, and thus the generative modeling problem is inherently ill-posed without further structure. Already, efforts have been made to achieve uniqueness by constraining the functional form of $f$, as in diffusion models~\cite{ho2020denoising}, or by only allowing certain paths of transport, as is the case in flow matching or methods based on optimal transport \cite{lipman2022flow}. Our theoretical results motivate a new way forward, in which the class of admissible models is constrained by the pushforward action of the generative model on a finite family of measures $\{\rho_j\}_{j=1}^{m}\subseteq \mathcal{P}(M)$. 

This collection can arise naturally in time-dependent physical modeling problems, as described in Sections \ref{subsec:dynamics} and \ref{subsec:PDE}, or may be introduced only to promote uniqueness. For example, as motivated by Theorem \ref{thm1}, rather than considering one reference noise distribution $\rho$, one can construct a finite collection $(\rho_1,\dots,\rho_{m})\in \boldsymbol{D}$, and instead enforce  $f_\#\rho_j = \nu_j$ for $1\leq j \leq m$. Another option, motivated by Corollary \ref{cor:CE}, involves constructing $f$ as the flow map of a continuity equation which interpolates a \textit{finite} family of probability measures at prescribed time marginals. While our theory guarantees uniqueness in these settings, an important direction of future study involves exploring how the existence of such generative models depends on the chosen pushforward constraints.

When the generative modeling procedure yields a unique pushforward map, we can represent it as a function $\mathsf{F}:\mathcal{P}(M)^{m} \to \textup{Diff}^1(M,M)$, where given the target measures  $\nu = (\nu_1,\dots, \nu_{m})\in \mathcal{P}(M)^{m}$, the output $f = \mathsf{F}(\nu)$ satisfies $f_\#\rho_j = \nu_j$ for $1\leq j \leq m$. A similar formulation also exists in the time-dependent setting. Quantifying how the generative model changes under perturbations of the data $\nu$ is important for understanding generalization ability, robustness to noise, and finite-sample complexity. 

Concretely, one may be interested in bounds of the form 
\begin{equation}\label{eq:stable}
    d(\mathsf{F}(\nu),\mathsf{F}(\nu^{\star})) \leq \Theta (\mathcal{D}_{m}( \nu,\nu^{\star})),
\end{equation}
 where $d(\cdot,\cdot)$ is a metric on $\textup{Diff}^1(M,M)$ and $\mathcal{D}_{m}$ is a metric on $\mathcal{P}(M)^{m}$, and $\Theta(\cdot)$ quantifies the type of stability, e.g., logarithmic, linear, Lipschitz, or H\"older. Understanding how the stability quantified by \eqref{eq:stable} depends on $\mathsf{F}$, $d$, and $\mathcal{D}_{m}$ can improve our understanding of existing models and inform the design of new architectures. 
 
When $(\rho_1,\dots,\rho_{m})\in \boldsymbol{D}$, the generic set introduced in Theorem \ref{thm1}, our  main theory developed in Section \ref{sec:results}, already provides insights regarding stability bounds of the form \eqref{eq:stable}. Towards this, suppose that $f,g\in \textup{Diff}^1(M,M)$ satisfy $f = \mathsf{F}(\nu)$ and $g = \mathsf{F}(\nu^{\star})$, i.e., $f_\#\rho_j = \nu_j$ and $g_\#\rho_j = \nu_j^{\star}$ for $1\leq j \leq m$. In this situation, we assume the target measures $\nu_j$ have been perturbed but that the reference measures $\rho_j$ remain fixed. It then follows that
 \begin{equation}\label{eq:stab_final}
     \mathfrak{D}(\mathsf{F}(\nu),\mathsf{F}(\nu^{\star})) = \sum_{j=1}^{m} \mathcal{D}(f_\#\rho_j,g_\#\rho_j) = \sum_{j=1}^{m} \mathcal{D}(\nu_j,\nu_j^{\star}) = \mathcal{D}_{m}(\nu,\nu^{\star}),
 \end{equation}
 which shows a Lipschitz stability bound of the form \eqref{eq:stable} with $\Theta (x) = x$. In \eqref{eq:stab_final}, $\mathfrak{D}(\cdot,\cdot)$ is the metric on $\textup{Diff}^1(M,M)$ introduced in Corollary \ref{cor1}, $\mathcal{D}(\cdot,\cdot)$ is a metric over $\mathcal{P}(M)$, and $\mathcal{D}_{m}(\cdot,\cdot)$ is the corresponding metric over $\mathcal{P}(M)^{m}$ defined by summing $\mathcal{D}$ over each index. 
 
The stability bound \eqref{eq:stab_final} depends on the choice of metric $\mathcal{D}(\cdot,\cdot)$ over the space of probability measures \cite{li2024stochastic}, which in turn determines both $\mathfrak{D}(\cdot,\cdot)$ and $\mathcal{D}_{m}(\cdot,\cdot)$. To further understand the implications of \eqref{eq:stab_final}, it is important to study how the metric $\mathfrak{D}(\cdot,\cdot)$ we introduced in Corollary \ref{cor1} is related to other common notions of distance over $\textup{Diff}^1(M, M)$, and how this relationship depends on $\mathcal{D}(\cdot,\cdot)$.

\section{Proof of Results}\label{subsec:proofs}
This section contains the complete proofs of all theorems, corollaries, and propositions stated in Sections \ref{sec:results} and \ref{subsec:discussion}.
\subsection{Proofs for Section \ref{sec:results}} \label{subsec:proof_results}
We now present the proofs of our main results. Throughout Section \ref{subsec:proof_results}, $m > 2d+1$ is fixed. We begin by establishing several crucial lemmas, which are dedicated to proving the existence of the generic set $\boldsymbol{D}$, appearing in Theorem~\ref{thm1}.  A key ingredient in our construction of $\boldsymbol{D}$ involves taking preimages and forward images of known generic sets, e.g., $\boldsymbol{W}_{m-1}$ appearing in Theorem \ref{thm:whitney}, under suitable functions. Lemmas \ref{lemma:prod}, \ref{lemma:2}, and~\ref{lemma:I} establish the functions and their regularity properties we use to accomplish this.
\begin{lemma}\label{lemma:prod}
 Define
\[
\mathcal{F}: C_+^{1}(M,\mathbb{R}^m)\to C_+^{1}(M,\mathbb{R}^m)
\]
by
\begin{equation}\label{eq:F}  
\mathcal{F}[Y](x)
=\bigl(Y_1(x)Y_m(x),\,\ldots,\, Y_{m-1}(x)Y_m(x),\, Y_m(x)\bigr),
\qquad x\in M,
\end{equation}
for $Y(x)=(Y_1(x),\ldots,Y_m(x))\in C_+^{1}(M,\mathbb{R}^m)$. Then $\mathcal{F}$ is a homeomorphism.
\end{lemma}
\begin{proof}
    First, note that $\mathcal{F}$ is invertible with inverse $\mathcal{F}^{-1}:C_+^1(M,\mathbb{R}^m)\to C_+^1(M,\mathbb{R}^m)$ given by
    \begin{equation}\label{eq:1}
        \mathcal{F}^{-1}[Y](x) = \Bigg( \frac{Y_1(x)}{Y_m(x)},\ldots, \frac{Y_{m-1}(x)}{Y_m(x)},Y_m(x)\Bigg), \qquad x\in M.
    \end{equation}
    Moreover, since pointwise multiplication and division are continuous operations over $C^{1}_+(M,\mathbb{R})$, we have $\mathcal{F}$ and $\mathcal{F}^{-1}$ are both continuous over $C_+^1(M,\mathbb{R}^m)$. Thus, $\mathcal{F}$ is a homeomorphism.
\end{proof}
\begin{lemma}\label{lemma:2}
Consider the coordinate projection
\[
\pi : C_+^{1}(M,\mathbb{R}^{m}) \to C_+^{1}(M,\mathbb{R}^{m-1}),
\]
defined for $Y(x)=(Y_1(x),\ldots,Y_m(x))\in C_+^{1}(M,\mathbb{R}^m)$ by
\begin{equation}\label{eq:proj}
\pi[Y](x) = \bigl(Y_1(x),\ldots, Y_{m-1}(x)\bigr), \qquad x\in M.
\end{equation}
Then $\pi$ is continuous and open.
\end{lemma}

\begin{proof}
The continuity of $\pi$ is immediate by its definition \eqref{eq:proj}.

 We now define the map $\tilde{\pi}:C^1(M,\mathbb{R}^m)\to C^1(M,\mathbb{R}^{m-1})$, given by $\tilde{\pi}[Y] = (Y_1,\dots,Y_{m-1})$. Since $C^1$ is a Banach space with norm \eqref{eq:normC1} and the map $\tilde{\pi}$ is continuous, surjective, and linear, it follows by the open mapping theorem that $\tilde{\pi}$ is an open map. Since $C_+^1(M,\mathbb{R}^m)$ is open in $C^1(M,\mathbb{R}^m)$ it then follows that $\pi$ is also an open mapping.

  Indeed, for any open set $O \subseteq C_+^1(M,\mathbb{R}^m)$ it holds that $O$ is also open in $C^1(M,\mathbb{R}^m)$. Thus, $\pi(O)=\tilde{\pi}(O)$ is open in $C^1(M,\mathbb{R}^{m-1})$. 
  Because $\pi(O)\subseteq C_+^1(M,\mathbb{R}^{m-1})$ and $C_+^1(M,\mathbb{R}^{m-1})$ is open in $C^1(M,\mathbb{R}^{m-1})$, it follows that $\pi(O)$ is open in $C_+^1(M,\mathbb{R}^{m-1})$ with the subspace topology. Therefore $\pi$ is open. 
\end{proof}
 \begin{lemma}\label{lemma:I}
 Define
\[
\mathcal{I}: C_+^{1}(M,\mathbb{R}^m)\to D_+^{1}(M,\mathbb{R}^m)
\]
by
\begin{equation}\label{eq:I}
        \mathcal{I}[Y](x) = \Bigg( \frac{Y_1(x)}{\int_M Y_1(z) \, \mathrm{d}z},\ldots ,\frac{Y_{m}(x)}{\int_M Y_m(z) \, \mathrm{d}z}\Bigg),\qquad x\in M,
    \end{equation}
for $Y(x)=(Y_1(x),\ldots,Y_m(x))\in C_+^{1}(M,\mathbb{R}^m)$.  
Then $\mathcal{I}$ is continuous and surjective.
\end{lemma}
\begin{proof} 
 We will first establish that the map $\mathcal{I}$ is continuous. To prove this, it suffices to establish continuity of the map $y\mapsto y/\int_M y \, \textrm{d}x$ over the domain $C_+^1(M,\mathbb{R})$ with respect to the norm \eqref{eq:normC1}. Towards this, we will assume that $y_\ell\in C_+^1(M,\mathbb{R})$ for each $\ell\in \mathbb{N}$ and that $y_\ell\xrightarrow[]{\ell \to \infty} y\in C_+^1(M,\mathbb{R}).$ Then, we have that  
    \begin{align}
       \bigg \| \frac{y}{\int_M y \, \textrm{d}x}- \frac{y_\ell}{\int_M y_\ell \, \textrm{d}x }\bigg\|_{C^1} &\leq \bigg \| \frac{y-y_\ell}{\int_M y \, \textrm{d}x}\bigg\|_{C^1}+ \bigg\|y_\ell \bigg( \frac{1}{\int_M y\, \textrm{d}x} - \frac{1}{\int_M y_\ell\, \textrm{d}x}\bigg)\bigg\|_{C^1} \nonumber\\ 
       &\leq \frac{1}{\int_M y \, \textrm{d}x} \|y-y_\ell\|_{C^1}+\|y_\ell\|_{C^1}\bigg|  \frac{1}{\int_M y\, \textrm{d}x} - \frac{1}{\int_M y_\ell\, \textrm{d}x}\bigg|\xrightarrow[]{\ell\to \infty} 0 \label{eq:c2},
    \end{align}
    where \eqref{eq:c2} follows from the assumption that $y_\ell\xrightarrow[]{\ell\to \infty} y$ in $C^1$, which allows us to deduce convergence of the integrals $\int_M y_\ell \, \textrm{d}x \xrightarrow[]{\ell\to\infty} \int_M y \, \textrm{d}x$. Thus the proof of continuity is complete.
    
    The fact that $\mathcal{I}$ is surjective is clear from the definition of $\mathcal{I}$, since $\mathcal{I}[Y] = Y$ if $Y\in D_+^1(M,\mathbb{R}^m)$ and $D_+^1(M,\mathbb{R}^m)\subseteq C_+^1(M,\mathbb{R}^m)$. 
\end{proof}
The following lemma establishes conditions under which density is preserved under forward and inverse images. While the result is standard, it plays a crucial role in our analysis, so we include a quick proof for completeness.
\begin{lemma}\label{lemma:dense_open}
    Let $X$ and $Y$ be topological spaces, and let $f:X\to Y$ be continuous. Then:
\begin{enumerate}
    \item[(i)] If $B\subseteq Y$ is dense and $f$ is open, then $f^{-1}(B)$ is dense in $X$.
    \item[(ii)] If $A\subseteq X$ is dense and $f$ is surjective, then $f(A)$ is dense in $Y$.
\end{enumerate}
\end{lemma}
\begin{proof}

\textit{(i)}\; It suffices to show that $f^{-1}(B)$ intersects every non-empty open
subset of $X$. Let $U\subseteq X$ be a non-empty open set. Since $f$ is open,
$f(U)$ is an open subset of $Y$. Because $B$ is dense in $Y$, we have
$B\cap f(U)\neq \emptyset$, so choose $y\in B\cap f(U)$. By definition of
$f(U)$, there exists $x\in U$ with $f(x)=y$. Hence, $x\in U\cap f^{-1}(B)$.
As $U$ was arbitrary, it follows that $f^{-1}(B)$ is dense in $X$. 
\textit{(ii)}\; By \cite[Theorem~2.9(c)]{armstrong2013basic}, one has
$f(\overline{A}) \subseteq \overline{f(A)}$. Since $A$ is dense in $X$, we have
$\overline{A}=X$. Because $f$ is surjective, $f(X)=Y$. Hence, $\overline{f(A)} = Y$ and $f(A)$ is dense in $Y$.
\end{proof}

In Lemma~\ref{lemma:Q} below, we introduce a new generic set $\boldsymbol{Q}$ that is closely
related to the generic set $\boldsymbol{D}$ appearing in the proof of
Theorem~\ref{thm1}.

\begin{lemma}\label{lemma:Q}
The set \begin{equation}\label{eq:density}
      \boldsymbol{Q} := \displaystyle\bigg\{ Y\in C_+^1(M,\mathbb{R}^{m}) : \bigg( \frac{Y_1}{Y_{m}},\ldots ,\frac{Y_{m-1}}{Y_{m}}\bigg) \textup{ is an embedding}\bigg\}
   \end{equation} is open and dense in $C_+^1(M,\mathbb{R}^{m})$.\end{lemma}
   \begin{proof}
        We begin by defining the set 
    \begin{equation}\label{eq:whitney_pos}
        \boldsymbol{W}_+:=\{ Y\in C_+^1(M,\mathbb{R}^{m-1}): Y \text{ is an embedding}\}.
    \end{equation}
Note that 
$
\boldsymbol{W}_+ \;=\; C_+^{1}(M,\mathbb{R}^{m-1})\cap \boldsymbol{W}_{m-1},
$
where $\boldsymbol{W}_{m-1}$ is the set defined in Theorem~\ref{thm:whitney}.  
Since $C_+^{1}(M,\mathbb{R}^{m-1})\subseteq C^{1}(M,\mathbb{R}^{m-1})$ is open, 
Theorem~\ref{thm:whitney} implies that $\boldsymbol{W}_+$ is open and dense in 
$C_+^{1}(M,\mathbb{R}^{m-1})$.

Define a map 
$
\Lambda : C_+^{1}(M,\mathbb{R}^{m}) \to C_+^{1}(M,\mathbb{R}^{m-1})
$
by 
$
\Lambda := \pi \circ \mathcal{F}^{-1},
$
where $\mathcal{F}^{-1}$  and $\pi$ are given in \eqref{eq:1} and \eqref{eq:proj}, respectively. Observe that 
$
\boldsymbol{Q} = \Lambda^{-1}(\boldsymbol{W}_+).
$
By Lemmas~\ref{lemma:prod} and~\ref{lemma:2}, the map $\Lambda$ is continuous and open. 
Therefore, Lemma~\ref{lemma:dense_open}\textit{(i)} yields that $\boldsymbol{Q}$ is open and dense.
\end{proof}

At several points in the proofs of our main results, we scale embeddings or
compose them with auxiliary mappings. The following lemma shows that these
modified functions remain embeddings.
   \begin{lemma}\label{lemma:comp_embedding}
     Let $n\geq 1$.   If $Y=(Y_1,\ldots,Y_n)\in C_+^1(M,\mathbb{R}^n)$ is an embedding, then the following maps defined over $M$ are also embeddings:
       \begin{enumerate}
           \item[(i)] $Z(x):=(c_1Y_1(x),\ldots, c_k Y_n(x))$  where $c\in \mathbb{R}_{>0}^{n}$.
           \item[(ii)] $H(x):=\log(Y(x))=(\log(Y_1(x)),\ldots, \log(Y_n(x)))$.
       \end{enumerate}
   \end{lemma}
   \begin{proof}
       \textit{(i)}: We can rewrite $Z = L \circ Y$, where $L:\mathbb{R}_{> 0}^n\to \mathbb{R}_{> 0}^n$ is given by $L(x) = \text{diag}(c) x$ for $x\in \mathbb{R}_{> 0}^n$. One can check that $L$ is a diffeomorphism of $\mathbb{R}_{>0}^n$. Since the composition of embeddings remains an embedding, the claim follows. \textit{(ii)}: We will define $S(x):=(\log(x_1),\ldots, \log(x_n))$ for $x\in \mathbb{R}_{> 0}^n$. Again it is quick to verify that $S$ is a diffeomorphism between $\mathbb{R}_{>0}^n$ and $\mathbb{R}^n$. Therefore, $H$, a composition of embeddings, remains an embedding.
   \end{proof}

   We next show that if the delay map based upon $h$ is an embedding that the delay map based upon $h^{-1}$ must also be an embedding. This is needed for our proof of Theorem \ref{thm2}.
   \begin{lemma}\label{lemma:permute_embed}
       Let $n\geq 2d+1$, let $(y,h)\in C^1(M,\mathbb{R})\times \textup{Diff}^1(M,M)$, and assume that $\Psi_{(y,h)}^{(n)}$ is an embedding. Then, $\Psi_{(y,h^{-1})}^{(n)}$ is also an embedding. 
   \end{lemma}
   \begin{proof}
       Define the permutation map $\sigma: \mathbb{R}^n \to \mathbb{R}^n$ by setting $$\sigma(x_1,\ldots, x_{n-1}, x_n) := (x_n, x_{n-1},\ldots, x_1),\qquad (x_1,\ldots,x_{n})\in \mathbb{R}^{n}.$$
       Note that $\sigma$ is invertible and linear, and therefore a diffeomorphism of $\mathbb{R}^n$. Next, we have for all $x\in M$ that 
\begin{align*}
    (\sigma \circ \Psi_{(y,h)}^{(n)}\circ (h^{n-1})^{-1})(x) &=  \sigma  \Big(y((h^{n-1})^{-1}(x)), \ldots, y(h^{-1}(x)),y(x)\Big)\\
    &= \Big(y(x), y(h^{-1}(x)),\ldots, y((h^{n-1})^{-1}(x))\Big)\\
    &= \Psi_{(y,h^{-1})}^{(n)}(x).
\end{align*}
Since $(h^{n-1})^{-1}$ is a diffeomorphism, $\Psi_{(y,h)}^{(n)}$ is an embedding, and $\sigma$ is a diffeomorphism, we have that $\Psi_{(y,h^{-1})}^{(n)}$ is an embedding.
   \end{proof}
     Next, we specify the generic set $\boldsymbol{D}$ appearing in the statement of Theorem~\ref{thm1}.
\begin{lemma}\label{lemma:thm1_set}
The set 
    \begin{equation}\label{eq:density2}
      \boldsymbol{D} := \displaystyle\bigg\{ Y = (Y_1,\ldots,Y_{m})\in D_+^1(M,\mathbb{R}^{m}) : \bigg( \frac{Y_1}{Y_{m}},\ldots ,\frac{Y_{m-1}}{Y_{m}}\bigg) \textup{ is an embedding}\bigg\}
      \end{equation}
      is open and dense in $D_+^1(M,\mathbb{R}^{m})$.
\end{lemma}
\begin{proof}
We have that $\boldsymbol{D} = \boldsymbol{Q} \cap D_+^1(M,\mathbb{R}^{m})$ and since $\boldsymbol{Q}$ is open (see Lemma~\ref{lemma:Q} and Equation~\eqref{eq:density}), we have by the definition of the subspace topology that $\boldsymbol{D}$ is open in $D_+^1(M,\mathbb{R}^{m})$.

    Next, we aim to show that $\boldsymbol{D}$ is also dense. Towards this, we will first show $\boldsymbol{D} = \mathcal{I}(\boldsymbol{Q})$, where $\mathcal{I}$ is defined in~\eqref{eq:I}. We begin by proving $\boldsymbol{D} \subseteq \mathcal{I}(\boldsymbol{Q})$. Let $Y\in \boldsymbol{D}$. Then, since $\boldsymbol{D}\subseteq \boldsymbol{Q}$ and $\mathcal{I}(Y) = Y$ we have that $Y\in \mathcal{I}(\boldsymbol{Q})$. 
    
    We now show that $\boldsymbol{D}\supseteq \mathcal{I}(\boldsymbol{Q})$. If $Y\in \mathcal{I}(\boldsymbol{Q})$, we have $Y = \mathcal{I}(W)\in D_+^1(M,\mathbb{R}^{m})$ for some $W\in \boldsymbol{Q}\subseteq C_+^1(M,\mathbb{R}^{m})$. By Lemma~\ref{lemma:Q}, $W = (W_1,\ldots,W_{m})$ has the property that $(W_1/W_{m},\ldots, W_{m-1}/W_{m})$ is an embedding.  
    It therefore follows by Lemma \ref{lemma:comp_embedding}\textit{(i)} that 
\begin{align*}\label{eq:LHS_RHS}
   \bigg( \frac{Y_1}{Y_{m}},\ldots, \frac{Y_{m-1}}{Y_{m}}\bigg)& =  \bigg( \frac{\mathcal{I}[W]_1}{\mathcal{I}[W]_{m}},\ldots, \frac{\mathcal{I}[W]_{m-1}}{\mathcal{I}[W]_{m}}\bigg) \\ &= \bigg(\frac{\int_{M}W_{m}\, \textrm{d}x}{\int_{M}W_{1}\, \textrm{d}x} \cdot \frac{W_1 }{W_{m}},\ldots, \frac{\int_{M}W_{m}\, \textrm{d}x}{\int_{M}W_{m-1}\, \textrm{d}x}\cdot \frac{W_{m-1}}{W_{m}}\bigg)
\end{align*}
is an embedding as well, where $\mathcal{I}[W]_j$ denotes the $j$-th element in the vector-valued function $\mathcal{I}[W]$, $1\leq j \leq m$.  As a result, we have $Y \in \boldsymbol{D}$, as wanted. 

Recall from Lemma~\ref{lemma:I} that $\mathcal{I}$ is continuous and surjective, and moreover the fact that $\boldsymbol{Q}$ is dense in  $C_+^{1}(M,\mathbb{R}^{m})$.   Therefore, by Lemma~\ref{lemma:dense_open}\textit{(ii)}, it follows that  $\boldsymbol{D}$ is dense in $D_+^{1}(M,\mathbb{R}^{m})$, as claimed.
\end{proof}

At this point, all required lemmas have been established and we proceed to prove the main results of our paper, including Theorem \ref{thm1}, Theorem \ref{thm2}, and Corollary \ref{cor1}.
\begin{proof}[Proof of Theorem \ref{thm1}]
We impose the condition that $(\rho_1,\ldots, \rho_{m})\in \boldsymbol{D}$, which is defined in~\eqref{eq:density2}. Assuming $f_\#\rho_j = g_\#\rho_j$, we have by the change of variables formula \eqref{eq:cov} that
    \begin{equation}\label{eq:cov2}
        \rho_j(f^{-1}(x)) |\det df_x^{-1}| = \rho_j(g^{-1}(x))|\det dg_x^{-1}|,\qquad x\in N,\qquad 1\leq j \leq m.
    \end{equation}
    After dividing~\eqref{eq:cov2} where $1\leq j\leq m-1$ by~\eqref{eq:cov2} where $j=m$, we obtain \begin{equation}\label{eq:after_division}
        \frac{\rho_j(f^{-1}(x))}{\rho_{m}(f^{-1}(x))} = \frac{\rho_j(g^{-1}(x))}{\rho_{m}(g^{-1}(x))},\qquad x\in N,\qquad 1\leq j \leq m-1.
    \end{equation}
       Since $(\rho_1/\rho_{m},\ldots,\rho_{m-1}/\rho_{m})$ is an embedding by the definition of $\boldsymbol{D}$ in~\eqref{eq:density2}, \eqref{eq:after_division} implies $f^{-1}=g^{-1},$ and hence, $f = g.$ This proves~\hyperref[i]{(K1)} under the conditions in Theorem~\ref{thm1}.
       
       To prove~\hyperref[ii]{(K2)}, we again fix $(\rho_1,\ldots,\rho_{m})\in \boldsymbol{D}$ and assume that
   \begin{equation}\label{eq:div}
        \textup{div} (\rho_j v)(x) = \textup{div} (\rho_j w)(x)\qquad x\in M ,\qquad 1\leq j \leq m.\
   \end{equation}
   Using the product rule, Equation~\eqref{eq:div} rearranges into 
   \begin{equation}\label{eq:div2}
      \langle \nabla \rho_j(x),  v(x)\rangle_{r_x} +  \rho_j(x) \textup{div} (v)(x) = \langle\nabla \rho_j(x) , w(x)\rangle_{r_x} + \rho_j(x) \textup{div} (w)(x),\quad x\in M,
   \end{equation}
 for $ 1\leq j \leq m.$ Dividing \eqref{eq:div2} by $\rho_j$ gives 
   \begin{equation}\label{eq:div3}
     \langle  \nabla \log(\rho_j(x)) , v(x)\rangle_{r_x} + \textup{div} (v)(x) = \langle \nabla \log(\rho_j(x)), w(x)\rangle_{r_x} + \textup{div} (w)(x),\quad x\in M,
   \end{equation}
   for  $1\leq j \leq m.$ Here, we used that fact that
   \[
   \frac{1}{\rho_j(x)}  \left \langle \nabla \rho_j(x) , v(x) \right\rangle_{r_x} =  \left\langle \frac{1}{\rho_j(x)}  \nabla\rho_j(x), v(x)\right\rangle_{r_x} =  \left\langle    \nabla \log(\rho_j(x)) , v(x)\right\rangle_{r_x},
   \]
   since for a fixed $x\in M$, $\rho_j(x)$ is a scalar and the metric $r_x$ is bilinear.
 
 Next, we use \eqref{eq:div3} to simplify $\textup{div} (\rho_jv)(x) - \textup{div} (\rho_{m} v)(x)$, which yields 
   \begin{equation}\label{eq:div4}
    \bigg\langle \nabla \log\bigg( \frac{\rho_j(x)}{\rho_{m}(x)} \bigg), v(x)\bigg\rangle_{r_x} = \bigg\langle \nabla \log\bigg( \frac{\rho_j(x)}{\rho_{m}(x)} \bigg), w(x)\bigg\rangle_{r_x},\quad x\in M ,
   \end{equation}
   for $ 1\leq j \leq m-1.  $ Now define $\Phi: M \to \mathbb{R}^{m-1}$ by setting $$\Phi(x):= \bigg(\log\bigg( \frac{\rho_1(x)}{\rho_{m}(x)} \bigg),\ldots, \log \bigg(\frac{\rho_{m-1}(x)}{\rho_{m}(x)}\bigg)\bigg), \qquad x\in M$$
   and note by Lemma \ref{lemma:comp_embedding}\textit{(ii)} that $\Phi$ is an embedding. Rewriting \eqref{eq:div4} in vectorized notation we have that 
   $$d \Phi_x v = d\Phi_x w,\qquad x\in M.$$
   Since $\Phi$ is an embedding with injective differential it then follows that $v = w$, as desired. 
\end{proof}

\begin{proof}[Proof of Theorem \ref{thm2}]
In this proof we will write $J_f(x):= |\det d f_x|$ for notational simplicity.
 We will also denote $\psi:= \rho_1/h_\#\rho_1.$ Throughout, we assume that $(\psi,h) \in \boldsymbol{G}$, which is the generic set appearing in Theorem \ref{thm:takens}. Since $\rho_j = h^{j-1}_\#\rho_1$, we have that by the change of variables formula~\eqref{eq:cov} that
    \begin{equation}\label{eq:tak1}
        \rho_j(x)=\rho_1((h^{j-1})^{-1}(x))J_{(h^{j-1})^{-1}}(x), \qquad x\in M,
    \end{equation}
    for  $1\leq j \leq m.$ Using the chain rule $$J_{(h^{j})^{-1}}(x)=J_{h^{-1}}((h^{j-1})^{-1}(x)) J_{(h^{j-1})^{-1}}(x), \qquad x\in M,$$
     we then obtain
\begin{align}
    \frac{\rho_{j}(x)}{\rho_{j+1}(x)} = \frac{\rho_1((h^{j-1})^{-1}(x))J_{(h^{j-1})^{-1}}(x)}{\rho_1((h^{j})^{-1}(x))J_{(h^{j})^{-1}}(x)}  &= \frac{\rho_1((h^{j-1})^{-1}(x)) J_{(h^{j-1})^{-1}}(x)}{\rho_1((h^{j})^{-1}(x))  J_{h^{-1}}((h^{j-1})^{-1}(x)) J_{(h^{j-1})^{-1}}(x)}\label{eq:jac_1} \\
    &= \frac{\rho_1((h^{j-1})^{-1}(x))}{\rho_1((h^{j})^{-1}(x)) J_{h^{-1}}((h^{j-1})^{-1}(x))} \nonumber\\ &= \frac{\rho_1((h^{j-1})^{-1}(x))}{\rho_2((h^{j-1})^{-1}(x))}  \label{eq:jac_2} \\&= \psi((h^{j-1})^{-1}(x)), \label{eq:tak3}
\end{align}
for $1\leq j \leq m-1$. Above, \eqref{eq:jac_1} comes from \eqref{eq:tak1} and the chain rule, and  \eqref{eq:jac_2} again makes use of~\eqref{eq:tak1}. 

Now, let us first assume that $f_\# \rho_j = g_\#\rho_j$ for $1\leq j \leq m$. This means that 
\begin{equation}\label{eq:first}
    \rho_j(f^{-1}(x)) J_{f^{-1}}(x) = \rho_j(g^{-1}(x))J_{g^{-1}}(x).
\end{equation}
Dividing $f_\#\rho_j$ by $f_\#\rho_{j+1}$ and using \eqref{eq:first} to simplify then yields
\begin{equation}\label{eq:second}
    \frac{\rho_j(f^{-1}(x))}{\rho_{j+1}(f^{-1}(x))} =  \frac{\rho_j(g^{-1}(x))}{\rho_{j+1}(g^{-1}(x))}, \quad \text{for $1 \leq j \leq m-1$.}
\end{equation} 
Applying \eqref{eq:tak3}, we obtain 
$$\psi((h^{j-1})^{-1}(f^{-1}(x))) = \psi((h^{j-1})^{-1}(g^{-1}(x))),\quad \quad \text{for $1 \leq j \leq m-1$.}$$
Recalling the time-delay map from Definition \ref{def:delay_map}, we then have 
\begin{equation}\label{eq:finvginv}
\Psi^{(m-1)}_{(\psi,h^{-1})}\left(f^{-1}(x)\right) = \Psi^{(m-1)}_{(\psi,h^{-1})}\left(g^{-1}(x)\right),\qquad x\in N. 
\end{equation}
Note that by the assumption that $(\psi,h)\in \boldsymbol{G}$ (see Assumption~\ref{assumption:1}), we have $\Psi_{(\psi,h)}^{(m-1)}$ is an embedding as a result of Takens' theorem. Hence, it follows by Lemma \ref{lemma:permute_embed} that $\Psi_{(\psi,h^{-1})}^{(m-1)}$ is an embedding as well. 
Finally, combining Equation~\eqref{eq:finvginv} with the injectivity of an embedding, we have $f^{-1} = g^{-1}$, which gives $f = g$, as desired.\\
\\
We next consider the case when $\textup{div} (\rho_j v) = \textup{div} (\rho_j w)$ for $1\leq j \leq m$. Following similar steps in the proof of Theorem~\ref{thm1}, we obtain from this relationship that
   \begin{equation}\label{eq:v_eq1}
      \Big\langle \nabla \log(\rho_j(x)) , v(x)\Big\rangle_{r_x} + \textup{div} (v)(x) = \Big\langle \nabla \log(\rho_j(x)), w(x)\Big\rangle_{r_x}+ \textup{div} (w)(x), \qquad x\in M
   \end{equation}
   for $1\leq j \leq m$. Using \eqref{eq:v_eq1} to simplify $\textup{div} (\rho_j v) - \textup{div} (\rho_{j+1} v)$, we then obtain
   $$ \Bigg\langle \nabla \log \bigg( \frac{\rho_j(x)}{\rho_{j+1}(x)}\bigg) , v(x) \Bigg\rangle_{r_x} =  \Bigg\langle \nabla \log \bigg( \frac{\rho_j(x)}{\rho_{j+1}(x)}\bigg) , w(x) \Bigg\rangle_{r_x}, \qquad x\in M$$
   for $1\leq j \leq m-1$.
   Utilizing Equation~\eqref{eq:tak3}, the above equation becomes
   \begin{equation}\label{eq:vw}
    \Big\langle \nabla \log (\psi((h^{j-1})^{-1}(x))) , v(x)\Big\rangle_{r_x} = \Big\langle \nabla \log (\psi((h^{j-1})^{-1}(x))) , w(x) \Big\rangle_{r_x}, \quad \text{ $1\leq j \leq m-1$.}\end{equation}
   Let $\Phi(x):= \log\left( \Psi^{(m-1)}_{(\psi,h^{-1})}(x)\right)$, which is an embedding by Lemma~\ref{lemma:comp_embedding}. Thus, Equation~\eqref{eq:vw}  is equivalent to $d\Phi_xv = d\Phi_xw$, which implies $v = w$ and completes the proof.
\end{proof}

\begin{proof}[Proof of Corollary \ref{cor1}]
Assume that either $(\rho_1,\ldots,\rho_{m})\in \boldsymbol{D}$ or that $\rho_j = h_{\#}^{j-1} \rho_1$ for $1\leq j \leq m$ where $(\rho_1,h)$ satisfies Assumption~\ref{assumption:1}.  
To see that $\mathfrak{D}$ is a metric on $\mathrm{Diff}^1(M,N)$, note first that
$\mathfrak{D}(f,f)=0$. If $f\neq g$, then by Theorem \ref{thm1} and Theorem \ref{thm2}  there exists 
$j$ such that $f_\#\rho_j \neq g_\#\rho_j$, hence $\mathfrak{D}(f,g)>0$.  
Symmetry follows from the symmetry of $\mathcal{D}$, and the triangle
inequality follows immediately from the triangle inequality for $\mathcal{D}$:
\[
\mathfrak{D}(f,g)
=\sum_{j}\mathcal{D}(f_\#\rho_j,g_\#\rho_j)
\le \sum_{j}\Big(\mathcal{D}(f_\#\rho_j,q_\#\rho_j)
+\mathcal{D}(q_\#\rho_j,g_\#\rho_j)\Big)
= \mathfrak{D}(f,q)+\mathfrak{D}(q,g),
\]
for any $q \in \textup{Diff}^1(M,N)$.
Thus $\mathfrak{D}$ is a metric.

The argument for $\mathsf{D}$ is identical.  Clearly $\mathsf{D}(v,v)=0$, and if 
$v\neq w$, then by Theorem~\ref{thm1} and Theorem~\ref{thm2} there exists $j$ with 
$\textup{div}(\rho_j v)\neq \textup{div}(\rho_j w)$, so $\mathsf{D}(v,w)>0$.  
Symmetry follows from symmetry of $\mathbf{d}$, and the triangle inequality from that of $\mathbf{d}$. Hence, $\mathsf{D}$ is a metric. 
\end{proof}
\subsection{Proofs for Section \ref{subsec:dynamics}}
We now present the proofs of Propositions~\ref{prop:koopman1} and~\ref{prop:koopman2}, which are applications of Whitney and Takens embedding theorems.
\begin{proof}[Proof of Proposition \ref{prop:koopman1}]
    Let $\boldsymbol{W}_m$ be the generic set from Theorem \ref{thm:whitney} and choose $Y = (y_1,\ldots, y_m)\in \boldsymbol{W}_m.$ First, if $y_j\circ f = y_j \circ g$, this implies $Y \circ f = Y\circ g$ and since $Y$ is invertible we have $f =g$. Next, if $\langle \nabla y_j , v\rangle =\langle \nabla y_j , w\rangle$, we have that $dY_x (v) = dY_x (w)$ for all $x\in M$.  Since $dY_x$ is injective for each $x\in M$, this implies $v = w$. 
\end{proof}
\begin{proof}[Proof of Proposition \ref{prop:koopman2}]
    Let $\boldsymbol{G}$ be the generic set from Theorem \ref{thm:takens} and choose $(y,h)\in \boldsymbol{G}$. Recall that $y_j:= y\circ h^{j-1} $ for $1\leq j \leq m$. By Theorem \ref{thm:takens} we then have that the delay map
    $$\Psi_{(y,h)}^{(m)}(x):= (y(x),y(h(x)),\ldots,y(h^{m-1}(x))) = (y_1(x),\ldots, y_{m}(x)),\qquad x\in M$$
    is an embedding. As shorthand we will write $\Psi = \Psi_{(y,h)}^{(m)}.$ Since $y_j \circ f = y_j \circ g$ for $1\leq j \leq m$ we have $\Psi \circ f =\Psi \circ g $ and thus $f = g.$ Moreover, if $ \langle \nabla y_j , v\rangle =  \langle \nabla y_j , w \rangle$  for $1\leq j \leq m$ then $d\Psi_x (v) = d\Psi_x (w)$ and we obtain $v = w.$  
\end{proof}
The proof of Corollary \ref{cor2} is analogous to the proof of Corollary \ref{cor1}.
\begin{proof}[Proof of Corollary \ref{cor2}]
Assume that either $(y_1,\ldots,y_m)\in \boldsymbol{W}_m$ or $y_j = y_1\circ h^{j-1}$ for $1\leq j \leq m$ where $(y_1,h)\in \boldsymbol{G}$.  
By construction it is clear that
$\mathfrak{D}(f,f)=0$. If $f\neq g$, then by Propositions~\ref{prop:koopman1} and~\ref{prop:koopman2}  there exists 
$j$ such that $y_j\circ f \neq y_j\circ g$, hence $\mathfrak{D}(f,g)>0$.  
Symmetry follows from the symmetry of $\mathbf{d}$, and the triangle
inequality follows immediately from the triangle inequality for $\mathbf{d}$:
\[
\mathfrak{D}(f,g)
=\sum_{j}\mathbf{d}(y_j\circ f,y_j\circ g)
\le \sum_{j}\Big(\mathbf{d}(y_j\circ f,y_j\circ q)
+\mathbf{d}(y_j\circ q,y_j\circ g)\Big)
= \mathfrak{D}(f,q)+\mathfrak{D}(q,g),
\]
for any $q\in \textup{Diff}^1(M,M)$.
Thus $\mathfrak{D}$ is a metric.

The argument for $\mathsf{D}$ is identical.  Clearly $\mathsf{D}(v,v)=0$, and if 
$v\neq w$, then by Propositions~\ref{prop:koopman1}~and~\ref{prop:koopman2} there exists $j$ with 
$\langle \nabla y_j , v  \rangle \neq \langle \nabla y_j , w\rangle$, so $\mathsf{D}(v,w)>0$.  
Symmetry follows from symmetry of $\mathbf{d}$, and the triangle inequality from that of $\mathbf{d}$. Hence, $\mathsf{D}$ is a metric.
\end{proof}
\subsection{Proofs for Section \ref{subsec:PDE} }

We now present proofs of the results from Section \ref{subsec:PDE} which provide guarantees for unique vector field recovery from finite snapshot data in certain PDE inverse problems. In our proof of Corollary \ref{cor:CE}, we will use the fact that the solution $t\mapsto \rho(t)$ of the continuity equation $$\partial_t \rho(t,x) + \textup{div}(\rho(t,x) v(x)) = 0$$ over $[0,T]$ is given by $\rho(t) = (f_t)_\#\rho_0$ where $f_t$ is the time-$t$ flow map of the vector field $v$. We also use the fact that the solution $t\mapsto \rho(t)$ of the advection equation $$\partial_t\rho(t,x) + \langle \nabla \rho(t,x), v(x)\rangle = 0 $$ is given by $\rho(t) = \rho_0 \circ f_t^{-1}.$

\begin{proof}[Proof of Corollary \ref{cor:CE}] \textbf{(Continuity equation)}  First, we consider the case when $(\rho_0,f_{\Delta t}^{(v)})$ satisfies Assumption \ref{assumption:1} and where the solution maps $t\mapsto \rho^{(v)}(t)$ and $t\mapsto \rho^{(w)}(t)$ satisfy the continuity equation \eqref{eq:PDE_1} over $[0,T]$ with $\rho^{(v)}(0) = \rho_0 = \rho^{(w)}(0)$. We further assume that $\rho^{(v)}(t_j,x) = \rho^{(w)}(t_j,x)$ for $0\leq j \leq m$ and all $x\in M$.  Viewing the continuity equation solution operator as a pushforward map, this implies that $$\left(f_{t_j}^{(v)}\right)_\#\rho_0 =\left(f_{t_j}^{(w)}\right)_\#\rho_0,\qquad 0 \leq j \leq m.$$ Because measurements are uniform in time, i.e., $t_j = j\Delta t$, we equivalently have 
    \begin{equation}\label{eq:rhoj}
         \rho_j:=\left (f_{\Delta t}^{(v)}\right )^j_\#\rho_0 = \left (f_{\Delta t}^{(w)}\right)^j_\#\rho_0,\qquad  1 \leq j \leq m.
    \end{equation}
By the definition of $\rho_j$ in \eqref{eq:rhoj}, it also follows that 
$$
\left(f_{\Delta t}^{(v)}\right)_\#\rho_j = \left(f_{\Delta t}^{(w)}\right)_\#\rho_j,\qquad 0\leq j \leq m-1.$$
Since $(\rho_0,f_{\Delta t})$ satisfies Assumption \ref{assumption:1}, it then follows by Theorem \ref{thm2} that $f_{\Delta t}^{(v)}(x) = f_{\Delta t}^{(w)}(x)$ for all $x\in M$, as desired. If we have additional knowledge that $\mathcal{L}^{(v)}[\rho^{(v)}(t_j,x)]= \mathcal{L}^{(w)}[\rho^{(w)}(t_j,x)]$ for $0\leq j \leq m-1$ and all $x\in M$, then by definition of the continuity equation this gives that $\textup{div} ( \rho_j v)(x) = \textup{div} (\rho_j w)(x)$ for $0\leq j \leq m-1$ and all $x\in M$. The conclusion that $v(x) = w(x)$ for all $x\in M$ then follows from Theorem \ref{thm2}.\vspace{-0.2cm}\\
\\
\noindent \textbf{(Advection equation)} We now turn to the case when $(\rho_0,f_{\Delta t}^{(v)})\in \boldsymbol{G}$ and $t\mapsto \rho^{(v)}(t)$ and $t\mapsto \rho^{(w)}(t)$ satisfy the advection equation \eqref{eq:PDE_1} over $[0,T]$ with $\rho^{(v)}(t_j,x) = \rho^{(w)}(t_j,x)$ for $0\leq j \leq m$ and all $x\in M$. Using the fact that solutions to the advection equation are given by composition with the inverse flow map and, additionally, that measurements are uniform in time, we obtain 
\begin{equation}\label{eq:advs}
    \rho_j:= \rho_0 \circ \left ( f_{\Delta t}^{(v)} \right)^{-j} = \rho_0 \circ \left ( f_{\Delta t}^{(w)} \right)^{-j},\qquad 1 \leq j \leq m.
\end{equation} 
From the definition of $\rho_j$ in \eqref{eq:advs} it follows that
\begin{equation}\label{eq:advs2}
    \rho_j \circ
\left(f^{(v)}_{\Delta t}\right)^{-1} = \rho_j \circ \left(f^{(w)}_{\Delta t}\right)^{-1}, \qquad 0\leq j \leq m-1.
\end{equation}
Moreover, since $(\rho_0,f_{\Delta t}^{(v)})\in \boldsymbol{G}$, the generic set from Theorem \ref{thm:takens}, the delay map $\Psi_{(\rho_0,f_{\Delta t}^{(v)})}^{(m)}$ is an embedding. By Lemma \ref{lemma:permute_embed}, it then follows that $\Phi(x) =  (\rho_0(x),\dots,\rho_{m-1}(x))$, is an embedding. As a result, \eqref{eq:advs2} implies that $f_{\Delta t}^{(v)}(x) = f_{\Delta t}^{(w)}(x)$ for all $x\in M$. Moreover, if $\mathcal{L}^{(v)}[\rho^{(v)}(t_j,x)]= \mathcal{L}^{(w)}[\rho^{(w)}(t_j,x)]$ for $0\leq j \leq m-1$ and all $x\in M$, then by the definition of the advection equation we have $\langle \nabla \rho_j(x) ,v(x)\rangle_{r_x} = \langle \nabla \rho_j(x), w(x)\rangle_{r_x}$ for $0\leq j \leq m-1$ and all $x\in M$. This implies that $d\Phi_x v = d\Phi_x w$, and since $\Phi$ is an embedding, this implies by Definition \ref{def:embedding} that $v(x) = w(x)$ for all $x\in M$ as desired. 
\end{proof}

In the statement of Corollary \ref{cor:ADR}, we consider the action of the ADR differential operator $\mathcal{L}^{(v)}[\rho]$ on a $C^1$ density $\rho$. Given that the differential operator includes second-order spatial derivatives, we interpret this quantity in the weak sense, i.e., through the relationship
\begin{equation}\label{eq:weak}
    \int_M \phi \mathcal{L}^{(v)}[\rho] \, \textrm{d}x = \int_M \textup{div} (\rho v ) \phi \, \textrm{d}x + \int_M \langle \nabla \phi , D \nabla \rho\rangle\, \textrm{d}x + \int_M R(\rho)\phi \, \textrm{d}x,\quad \forall \phi \in C_c^{\infty}(M). 
\end{equation}
We now present the proof of Corollary \ref{cor:ADR}.
\begin{proof}[Proof of Corollary \ref{cor:ADR}]
    Assume that $(\rho_1,\dots,\rho_{m})\in \boldsymbol{D}$, the generic set from Theorem \ref{thm1} and that $\mathcal{L}^{(v)}[\rho_j] = \mathcal{L}^{(w)}[\rho_j]$ for $1\leq j \leq m$.  Equating the weak form expressions of these derivatives (see \eqref{eq:weak}) and canceling common terms,  we have 
    \begin{equation}\label{eq:int_eq}
        \int_M \textup{div}(\rho_j v) \phi \, \textrm{d}x = \int_M \textup{div}(\rho_j w) \phi \, \textrm{d}x ,\qquad \forall \phi \in C_c^{\infty}(M),\qquad 1\leq j \leq m.
    \end{equation}
    Note that the maps $x\mapsto \textup{div}(\rho_j v)(x)$ and $x\mapsto \textup{div}(\rho_j w)(x)$ are continuous, hence integrable, as $M$ is compact. Thus, since \eqref{eq:int_eq} holds for all $\phi\in C_c^{\infty}(M)$ we have that $\textup{div}(\rho_j v)(x) = \textup{div}(\rho_j w)(x)$ for all $x\in M$ and $1\leq j \leq m$.
    By Theorem \ref{thm1} this implies $v(x) = w(x)$ for all $x\in M$.
\end{proof}

\section{Numerical Experiments}\label{sec:numerics}

In this section, we conduct numerical tests  which showcase the unique recovery of transport maps and vector fields from finite measure-valued datasets.\footnote{Our code is available: \url{https://github.com/jrbotvinick/Transport-Map-and-Vector-Field-Recovery}.} Throughout, the unknown functions are parameterized by neural networks and the metrics introduced in Corollary \ref{cor1} are used to construct loss functions that promote unique training on distributional learning tasks. In Section \ref{subsec:numerics1d}, we  demonstrate unique pushforward map recovery in the time-independent setting. Section \ref{subsec:lorenz} then studies a corresponding time-dependent problem which can also be viewed as an inverse problem for the continuity equation. Across both Sections \ref{subsec:numerics1d} and \ref{subsec:lorenz} we repeat experiments with different randomized realizations of the measure-valued datasets, providing empirical evidence that the genericity assumptions introduced in Sections \ref{sec:results} and \ref{subsec:discussion} hold in practice. Finally, in Section \ref{subsec:vector} we demonstrate unique vector field recovery through the comparison of finitely many density-weighted divergence operators. We investigate how the accuracy of the vector field reconstruction depends on the number of available densities, and we discuss these results in light of the theoretical guarantees established in Section~\ref{sec:results}.

\subsection{Unique Recovery of a One-Dimensional Pushforward Map}\label{subsec:numerics1d}
We begin by showcasing unique pushforward map recovery for a simple one-dimensional example. We will aim to learn the map $f:\mathbb{S}^1\to \mathbb{R}^3$ given by
\begin{equation}\label{eq:smap}
    f(x) =  \Bigg( \sin(x),\; \frac{\cos(3x)+\sin(2x)}{2},\; \frac{\sin(3x)+\sin(5x)}{2} \Bigg), 
\qquad x \in [-\pi,\pi],
\end{equation}
from its pushforward action on five densities $\rho_1,\dots,\rho_5\in \mathcal{P}(\mathbb{S}^1)$. We have written $\mathbb{S}^1$ to denote the circle obtained by identifying the endpoints of $[-\pi,\pi]$. The densities $\rho_j$ are constructed from the von Mises distribution
\begin{equation*}\label{eq:VM}
  \rho_j(x;\alpha_j,\beta_j) \propto \exp( \alpha_j \cos(x - \beta_j)), \qquad x\in [-\pi,\pi], 
\end{equation*}
 where the concentrations $\{\alpha_j\}_{j=1}^5$ are sampled i.i.d.~from $\text{Unif}([1,3])$ and the centers $\{\beta_j\}_{j=1}^5$ are sampled i.i.d.~from $\text{Unif}([-\pi,\pi])$. These reference densities are visualized in the top row of Figure~\ref{fig:1a}. Thus, each $\rho_j$ is a realization of a random measure \cite{kallenberg2017random}.

\begin{figure}[h!]
  \centering
  
  \subfloat[(Top) input mesures $\rho_1,\ldots,\rho_5$. (Bottom) Marginals from output measures ${f} _\#\rho_i$ and ${\left(f_{\theta}\right)}_\#\rho_i$, all estimated from samples by a kernel density estimation.  ]
{\label{fig:1a}\includegraphics[width=\textwidth]{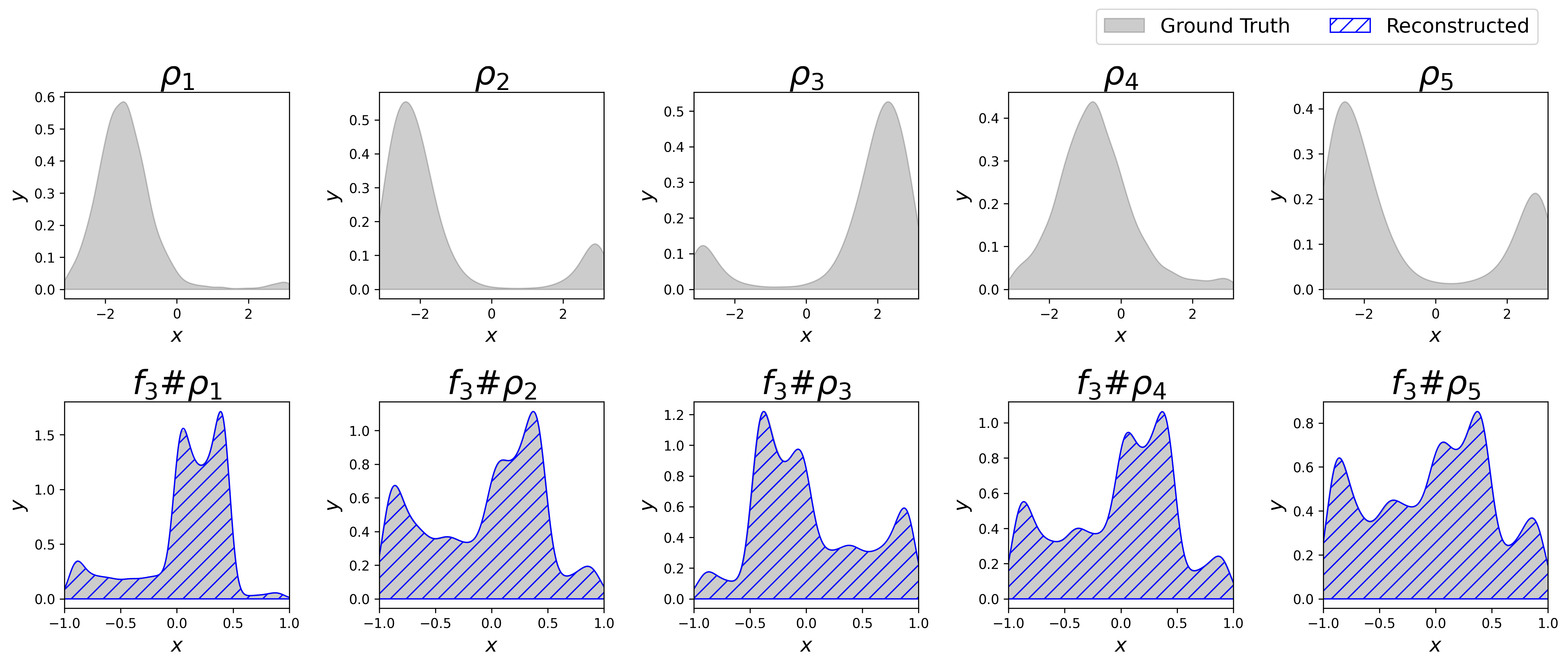}}
  \hfill
  \vspace{0.5cm}
  \subfloat[The function $f_{\theta}$ recovered from training on the measure-valued data closely agrees with the ground truth $f$.]
  {\label{fig:1b}\includegraphics[width=.75\textwidth]{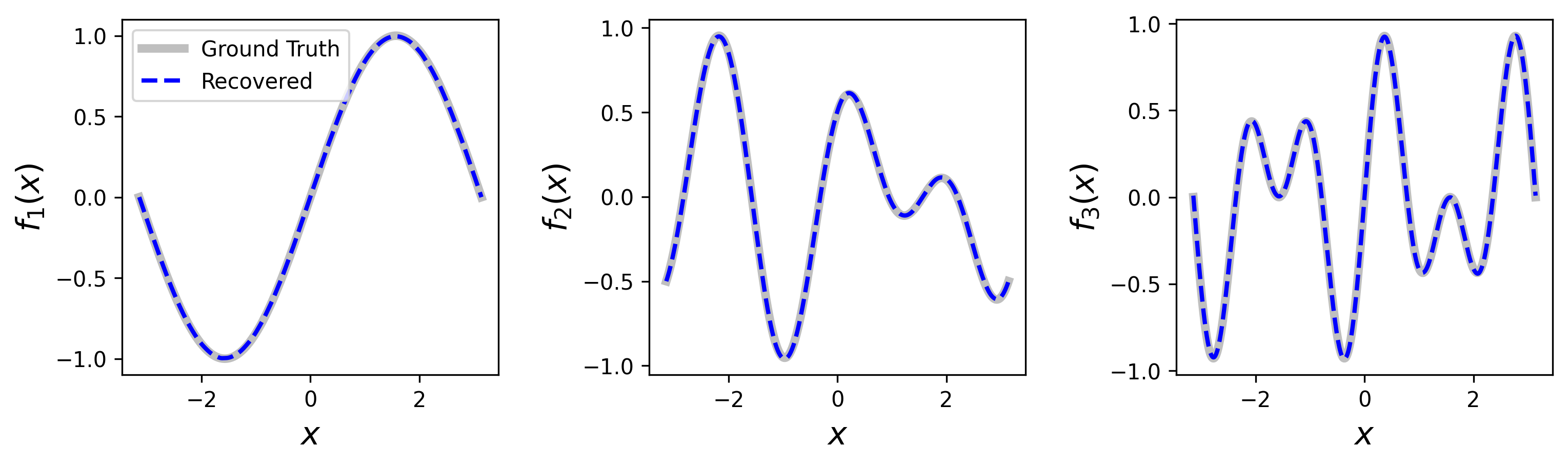}}
  \hfill
  \caption{Recovering a one-dimensional function from its pushforward action on five densities. }
  \label{fig:1}
\end{figure}

The densities $\rho_1,\dots,\rho_5$ are used to construct the training dataset

\begin{equation}\label{eq:dataset}
    \{(\rho_j, f_\# \rho_j)\}_{j=1}^5,
\end{equation}
where a marginal of the output measures $f_\#\rho_j$ is visualized in the bottom row of Figure~\ref{fig:1a}. Throughout Figure \ref{fig:1a}, we have written $f_i$ to denote the $i$-th component of the vector-valued map $f$. To learn $f$ from the dataset \eqref{eq:dataset}, we initialize a neural network $f_\theta:\mathbb{S}^1\to \mathbb{R}$ with weights and biases $\theta\in \mathbb{R}^p$ and seek to minimize the loss
\begin{equation}\label{eq:metric2}
    \mathcal{J}(\theta) = \frac{1}{5} \sum_{j=1}^5 \mathcal{D}\big(f_\# \rho_j,(f_{\theta})_\# \rho_j\big),
\end{equation}
where $\mathcal{D}$ denotes the energy Maximum Mean Discrepancy (MMD) \cite{feydy2019interpolating}. Note that the objective function~\eqref{eq:metric2} directly applies the metric introduced in Corollary \ref{cor1} to determine the mismatch between $f$ and $f_{\theta}$. 

The neural network  $f_{\theta}$ uses a Fourier embedding of $\mathbb{S}^1$ to ensure continuity and smoothness at the endpoints of $[-\pi,\pi]$. In particular, it is constructed as 
$$f_{\theta}(x) = h_{\theta}(\sin(x),\cos(x)),\qquad x\in [-\pi,\pi],$$
where $h_{\theta}:\mathbb{R}^2\to \mathbb{R}$ is a fully connected neural network with two hidden layers of $100$ nodes each and hyperbolic tangent activation function. At each optimization step, the loss \eqref{eq:metric2} is estimated using 100 i.i.d.~samples from each $\rho_j$ which are used to construct empirical approximations to $f_\#\rho_j$ and $(f_{\theta})_\#\rho_j$. Optimization is carried out using Adam \cite{DBLP:journals/corr/KingmaB14} with learning rate $10^{-3}$ for $5 \times 10^4$ iterations.

Following training, we visualize the learned map $f_\theta$; see Figure~\ref{fig:1b}. Consistent with the theoretical identifiability guarantees established in Theorem \ref{thm1}, the recovered $f_\theta$ closely matches the ground-truth  $f$. The experiment is repeated 10 times for different randomized choices of the measure-valued dataset. In each trial we randomly resample the concentrations and centers $\{(\alpha_j,\beta_j)\}_{j=1}^5$ which parameterize the densities $\rho_j$, as well as the initialization of the neural network. Following training, we assess the learned neural network's performance by approximating the relative mean squared error
\begin{equation*}\label{eq:rel_err}
    \text{MSE}(f_{\theta})=\frac{\int_{\mathbb{S}^1} |f(x) - f_{\theta}(x)|^2 \, \textrm{d}x}{ \int_{\mathbb{S}^1} |f(x)|^2 \, \textrm{d}x}.
\end{equation*} Across all ten trials, we found that the model $f_{\theta}$ converged to the ground truth solution with high accuracy with   $\text{MSE}(f_{\theta})$ taking values in $[1.54\cdot 10^{-5}, 8.36 \cdot 10^{-5}]$ with a median of $3.23\cdot 10^{-5}.$

\subsection{Unique Lorenz-63 Identification from Density Snapshots}\label{subsec:lorenz}

We next consider a more realistic setting in which data is generated by a time-dependent process. Here, we consider the Lorenz-63 system defined by the following system of differential equations 
\begin{equation}\label{eq:lorenz}\begin{cases}
    \dot{x} = \sigma(y-x)\\
    \dot{y} = x(\rho - z) - y\\
    \dot{z} = xy - \beta z
\end{cases}.\end{equation}
Denoting by $f$ the time-$\Delta t$ flow of \eqref{eq:lorenz} with $\Delta t = 0.1$, we consider a dataset of the form $\{\rho_j\}_{j=0}^{m}$ with $m = 7$ and $\rho_j := f^j_\#\rho$ for $0\leq j \leq m$, where $\rho\in \mathcal{P}(\mathbb{R}^3)$ is a Gaussian initial condition. The initial condition $\rho$ is a realization of a random measure with mean $(\gamma_1,\gamma_2,\gamma_3)$, where $\gamma_1 \sim \text{Unif}(-15,15)$, $\gamma_2 \sim \text{Unif}(-15,15)$, and $\gamma_3 \sim \text{Unif}(20,40)$, and covariance $\sigma^2 I $, where $\sigma \sim \text{Unif}(3,7)$ and $I\in \mathbb{R}^{3\times 3}$ denotes the identity matrix. 

\begin{figure}[h!]
  \centering
  
  \subfloat[Marginal of the training data onto the $x$-$z$ axis.]
  {\label{fig:3a}\includegraphics[width=\textwidth]{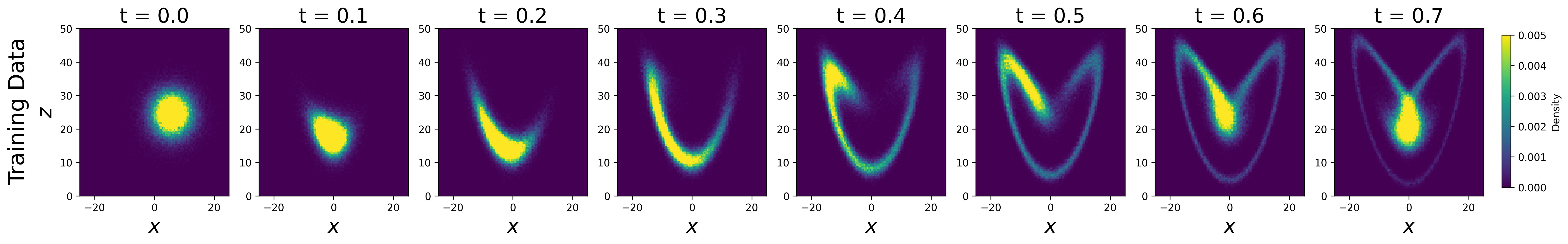}}
  \hfill
  \subfloat[Support of the observed data (top left), the relative reconstruction error (top right), the ground truth vector field (bottom left), and the learned vector field (bottom right).]
  {\label{fig:3b}\includegraphics[width=.7\textwidth]{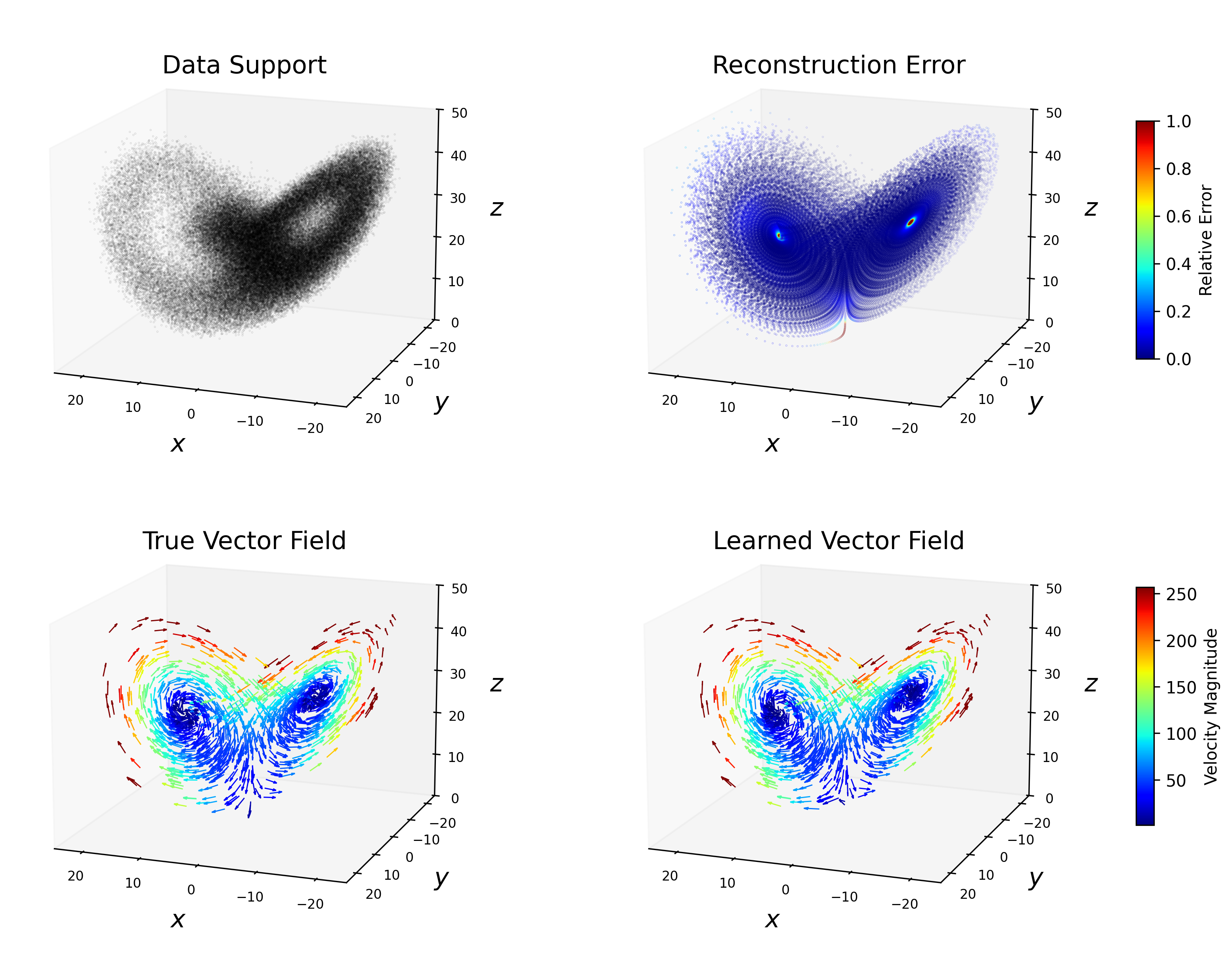}}
  \hfill
 \subfloat[Density evolution predicted by the learned model for a new Gaussian initial condition with mean $(-10,10,40)$ and covariance $ 49 I$.]
  {\label{fig:3c}\includegraphics[width=\textwidth]{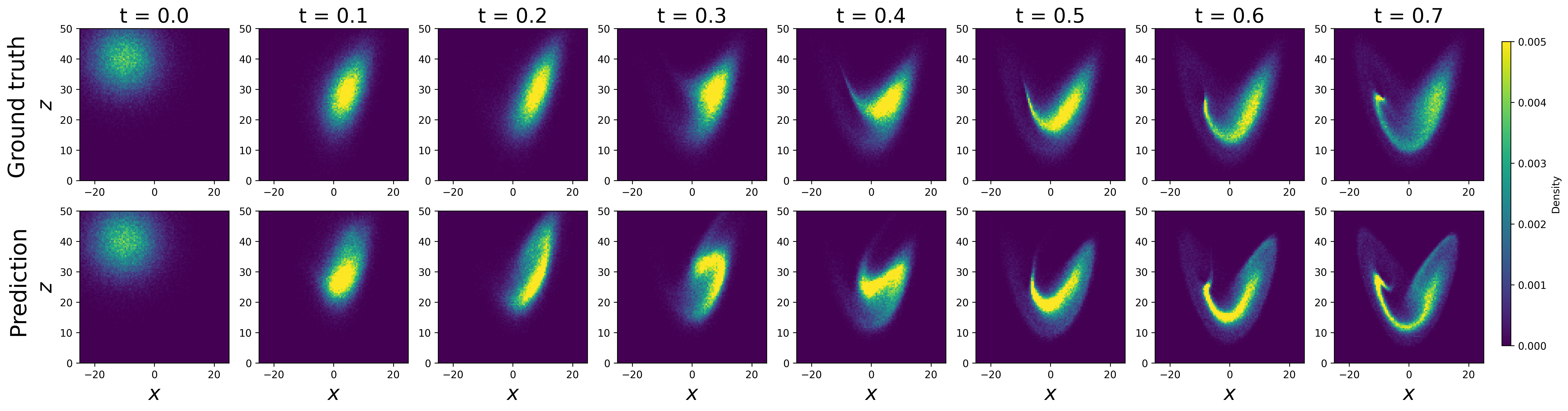}}
  \hfill
  \caption{Recovering the Lorenz-63 dynamics on the support of finite snapshot data $\{\rho_j\}_{j=0}^m$ that covers the full chaotic attractor. In this example, $\text{MSE}(v_{\theta}) = 1.48\cdot 10^{-2}$ and $\text{MSE}(f_{\theta}) = 2.27\cdot 10^{-3}.$}
  \label{fig:3}
\end{figure}

\begin{figure}[h!]
  \centering
  
  \subfloat[Marginal of the training data onto the $x$-$z$ axis.]
  {\label{fig:4a}\includegraphics[width=\textwidth]{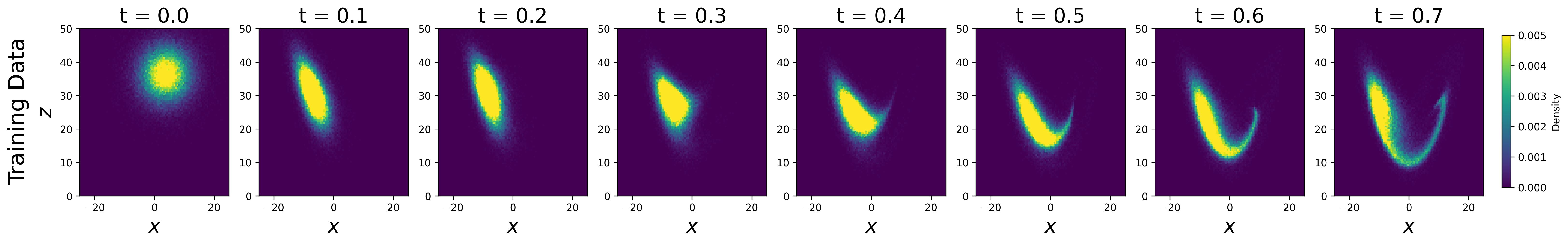}}
  \hfill
  \subfloat[Support of the observed data (top left), the relative reconstruction error (top right), the ground truth vector field (bottom left), and the learned vector field (bottom right).]
  {\label{fig:4b}\includegraphics[width=.7\textwidth]{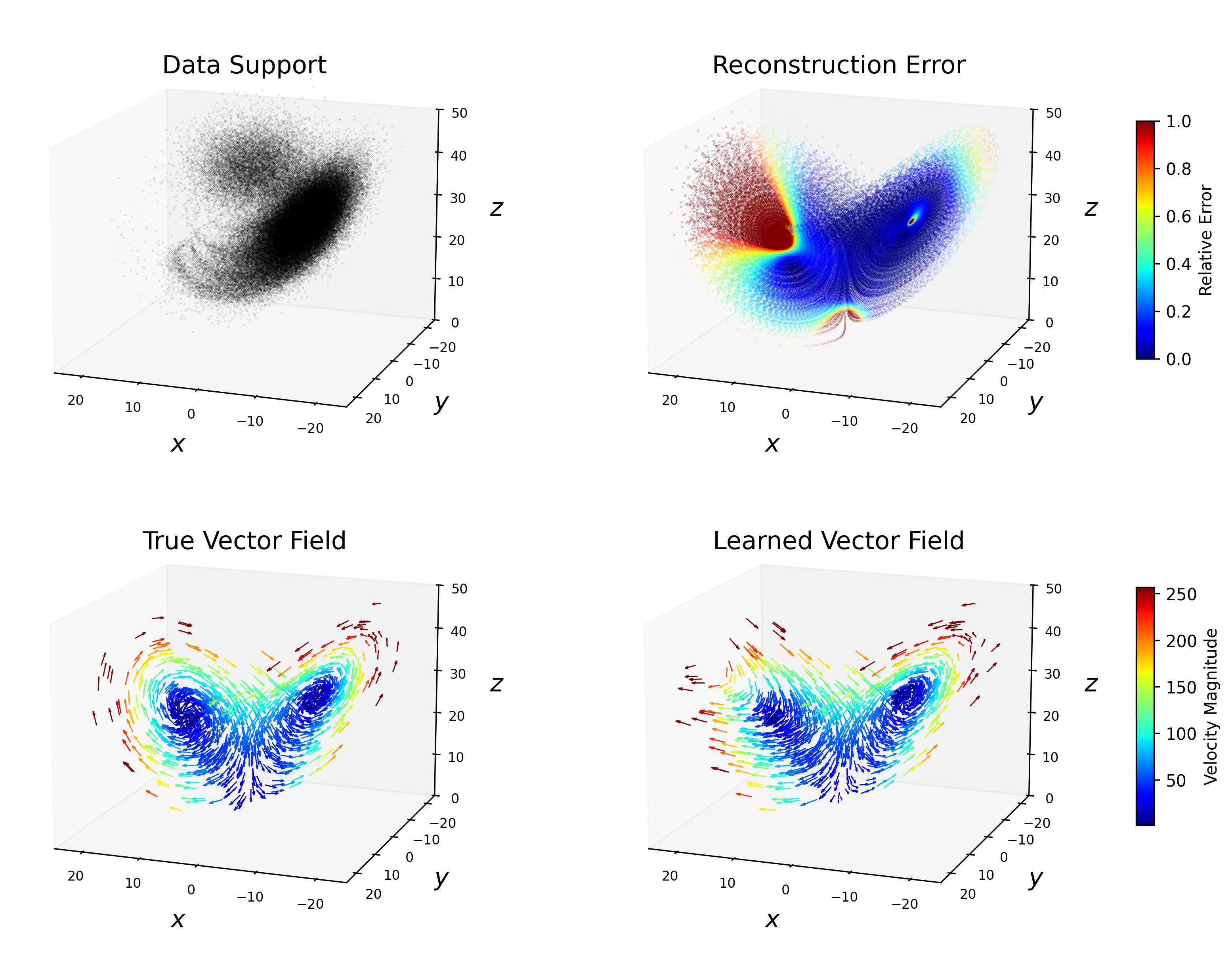}}
 \caption{Recovering the Lorenz-63 dynamics on the support of finite snapshot data $\{\rho_j\}_{j=0}^m$ which covers only a fraction of the chaotic attractor. The dynamics are accurately identified on the support of the observed density snapshots. In this example, $\text{MSE}(v_{\theta}) = 8.48\cdot 10^{-2}$ and $\text{MSE}(f_{\theta}) = 3.93\cdot 10^{-3}.$}\label{fig:4}
\end{figure}

Each measure $\rho_j$ is represented as an empirical distribution from $N = 10^5$ samples, and the pushforward $\rho_{j+1}=f_{\#}\rho_j$ is computed by evaluating the flow $f$, approximated via the forward Euler method with a time step of $0.01$, on each sample. We then seek to numerically recover the underlying vector field~\eqref{eq:lorenz} from this measure-valued dataset. In particular, we parameterize $v_{\theta}$ as a neural network with corresponding flow $f_{\theta}$ and seek to minimize the objective 
\begin{equation}\label{eq:lorenz_loss}
    \mathcal{J}(\theta) = \sum_{j=0}^{m-1}\mathcal{D}\left((f_{\theta})_\# \rho_j, \rho_{j+1}\right).
\end{equation}
Note that the objective \eqref{eq:lorenz_loss} directly applies the metric introduced in Corollary \ref{cor1} to solve this distribution-matching inverse problem. 

The neural network $v_{\theta}$ is fully connected with two hidden layers of 100 nodes and hyperbolic tangent activation, and the time-$\Delta t$ flow $f_{\theta}$ is again approximated using Euler's method with a timestep of $0.01$. Before training, all data is rescaled by an affine transformation to belong to the unit cube. We then use the Adam optimizer with a learning rate of $10^{-3}$ to minimize  $\mathcal{J}(\theta)$ over $10^4$ training iterations, where  $\mathcal{D}$ is the sample-based energy MMD and each term in the loss \eqref{eq:lorenz_loss} is approximated from minibatches of 200 sample points.  This experiment is repeated 10 times for different randomized means $(\gamma_1,\gamma_2,\gamma_3)$ and covariances $\sigma^2 I$ of the initial measure $\rho_0$, as well as different randomized neural network initializations. Following training, we assess the model's performance by evaluating both the relative mean-squared error of the vector field $v_{\theta}$ and flow map $f_{\theta}$, weighted by the measure
\begin{equation}\label{eq:mixture}
   \tilde{\rho}: = \frac{1}{m}\sum_{j=0}^{m-1} \rho_j \in \mathcal{P}(\mathbb{R}^3).
\end{equation}
In particular, our evaluation metrics for the trained model with parameters $\theta$ are given by 
\begin{equation*}\label{eq:error_metric_l}
    \text{MSE}(v_{\theta}) = \frac{ \int_{\mathbb{R}^3} | v_{\theta}(x) - v(x)|^2 \tilde{\rho}(x) \, \textrm{d}x }{\int_{\mathbb{R}^3} |v(x)|^2 \tilde{\rho}(x) \, \textrm{d}x }, \qquad  \text{MSE}(f_{\theta}) = \frac{ \int_{\mathbb{R}^3} | f_{\theta}(x) - f(x)|^2 \tilde{\rho}(x) \, \textrm{d}x }{\int_{\mathbb{R}^3} |f(x)|^2 \tilde{\rho}(x) \, \textrm{d}x },
\end{equation*}
where all integrals are approximated via Monte--Carlo sampling. By evaluating the relative mean squared errors weighted by $\tilde{\rho}$ we ensure that we only assess the model's performance in regions where data has been observed. Across all ten trials we observed accurate recovery of the vector field and flow map on the support of the data, with $\text{MSE}(v_{\theta})$ taking values in $[1.35\cdot 10^{-2}, 1.53\cdot 10^{-1}]$ with a median of $6.47\cdot 10^{-2}$, and with $\text{MSE}(f_{\theta})$ taking values in $[1.01\cdot 10^{-3}, 1.26\cdot 10^{-2}]$ with a median of $3.92\cdot 10^{-3}.$ These results are consistent with Corollary~\ref{cor:CE} where we introduced theoretical guarantees for recovering flow maps from finite snapshot data arising from continuity equations.

In Figures \ref{fig:3} and \ref{fig:4} we visualize two illustrative cases from these ten trials. Figure \ref{fig:3} depicts a situation in which the trajectory $\{\rho_j\}$ covers the full Lorenz-63 attractor, while in Figure \ref{fig:4} the trajectory only covers a fraction of the attractor. In both cases, we demonstrate accurate recovery of the vector field \textit{on the support of the observed data}. Figures \ref{fig:3a} and \ref{fig:4a} show marginal projections of the training data $\{\rho_j\}$. Moreover, in Figures \ref{fig:3b} and \ref{fig:4b} we visualize the support of $\tilde{\rho}$, the relative error $|v_{\theta}(x) - v(x)|^2/|v(x)|^2$ on the support of the Lorenz-63 attractor, the ground truth vector field $v$, and the learned vector field $v_{\theta}$. As expected, in Figure \ref{fig:4b} the support of $\tilde{\rho}$ is correlated with the regions of low relative error. 

In Figure \ref{fig:3c}, we show how the learned model can be used to conduct accurate forecasts for distributional dynamics  corresponding to new measure-valued initial conditions that were unseen during training. Extrapolating to new measure-valued initial conditions for the model learned in Figure \ref{fig:4} is more challenging, given that the support of the data $\{\rho_j\}_{j=0}^m$ did not include a critical part of the Lorenz-63 attractor. This limitation  arises solely from the observed data and is unrelated to the learning framework. 
\subsection{Unique Vector Field Recovery via Divergence Operator Comparison}\label{subsec:vector}
While Sections~\ref{subsec:numerics1d} and~\ref{subsec:lorenz} reconstruct transport maps by comparing pushforward measures, we now showcase an experiment in which a vector field is recovered by comparing weighted divergence operators over a finite collection of densities. In particular, we consider the 2D vector field defined by 
\begin{equation}\label{eq:2d_field}
    v(x,y) = (y,- \sin( 4 \pi x)), \qquad (x,y)\in [-1,1]^2,
\end{equation}
 which describes the motion of an undamped pendulum.
While we do not have direct access to the unknown ground truth $v$, we assume we can evaluate the function $x\mapsto \textup{div}( \rho_j v)(x)$ for a finite collection of densities $\{\rho_j\}_{j=1}^m$. Each $\rho_j$ is a realization of a random measure defined by $\rho_j = N(\gamma_j, \sigma_j^2 I)$ with $\gamma_j \sim \textup{Unif}([-1,1]^2)$ and $\sigma_j \sim \text{Unif}([0.75,1.25])$. We then parameterize the unknown vector field as a neural network $v_{\theta}:\mathbb{R}^2\to \mathbb{R}^2$ and seek to reduce the loss 
\begin{equation}\label{eq:lossdiv}
    \mathcal{J}(\theta) = \sum_{j=1}^m \|\textup{div}(\rho_j v) - \textup{div}(\rho_j v_{\theta})\|_{L^2}^2.
\end{equation}
The loss \eqref{eq:lossdiv} makes use of the metric on the space of vector fields introduced in Corollary~\ref{cor1}. We approximate $\mathcal{J}(\theta)$ using automatic differentiation and Monte--Carlo integration with randomly sampled minibatches of 200 points from $\text{Unif}([-1,1]^2)$. This follows the conventional training pipeline of a physics-informed neural network \cite{raissi2019physics}. The network $v_{\theta}$ has two hidden layers, each consisting of 50 nodes and using the hyperbolic tangent activation function, and optimization is carried out using Adam for $2\cdot 10^4$ iterations.

\begin{figure}[h!]
  \centering
  \subfloat[Reference vector field (left) and the relative error as a function of $m$  (right). ]
  {\includegraphics[width=.8\textwidth]{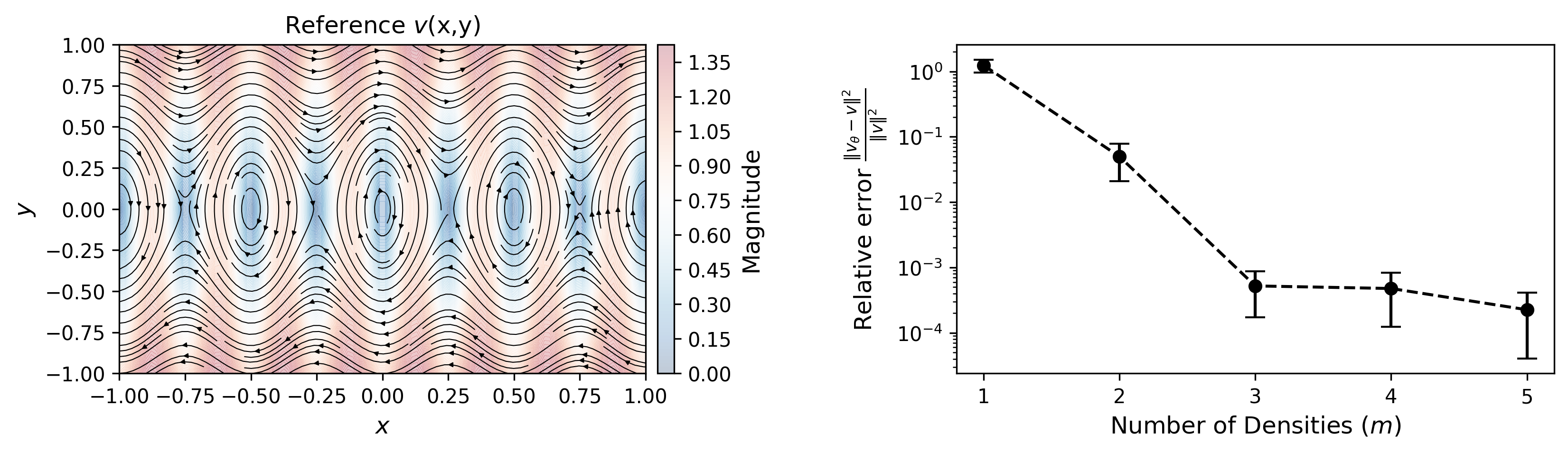}\label{fig:5a}}
  \vspace{0.2cm}
  \subfloat[Visualization of the densities $\{\rho_j\}_{j=1}^m$ via the function $\max_j \rho_j(x)$ (top row), the learned vector field $v_{\theta}$ after minimizing the objective \eqref{eq:lossdiv} (middle), and the absolute error between the learned $v_{\theta}$ and reference $v$ (bottom).]
  {\label{fig:5b}\includegraphics[width=\textwidth]{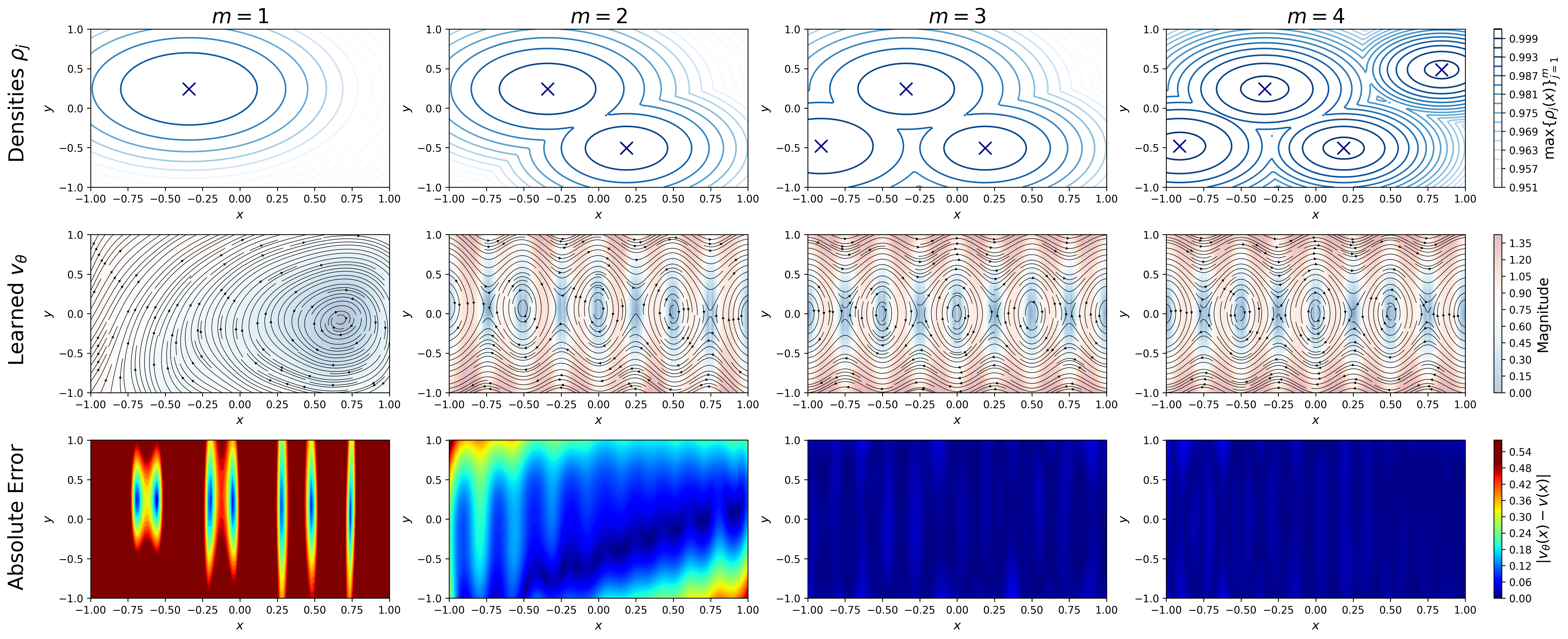}}
 \caption{Reconstructing the vector field \eqref{eq:2d_field} by comparing divergence operators via the objective function \eqref{eq:lossdiv}. The reference vector field $v$ and relative error as a function of $m$ are shown in Figure \ref{fig:5a}, while Figure \ref{fig:5b} visualizes individual training results for $m \in \{1,2,3,4\}$. }\label{fig:5}
\end{figure}

In Figure \ref{fig:5}, we report training results for various choices of $m$ to empirically investigate how the reconstructed vector field depends on the number of accessible weighted divergence operators.  For each fixed choice of $m$, we repeat the neural network training $10$ times with different randomized initializations.
For each network training we compute the relative error leading to the visualization of means and standard deviations in Figure \ref{fig:5a}. Figure \ref{fig:5b} visualizes individual training results for $m \in \{1,2,3,4\}$.

Theorem \ref{thm1} predicts that  unique recovery occurs when $ m > 2d+1$, which in this case is $m  =6$. This heuristic relies on the Whitney and Takens embedding theorems which are known to provide \textit{upper bounds} on the required embedding dimensions, while in practice lower dimensional embeddings often exist. These minimum embedding dimensions are expected to depend highly on the observed $\rho_j$ and underlying $v$. In Figure \ref{fig:5}, we observe accurate recovery of the underlying vector field using only $m = 3$ densities. 
\section{Conclusion}\label{sec:conclusion}

In this work, we established conditions guaranteeing that a diffeomorphism or vector field can be uniquely recovered from its action on finitely many probability measures. In the static setting, Theorem~\ref{thm1} proves that if $m>2d+1$, then for a generic family of $m$ strictly positive $C^1$ densities, equality of the corresponding pushforwards or weighted divergences forces equality of the underlying maps or vector fields. As a result, the theorem provides finite-data identifiability guarantees for transport-based inverse problems. We extended these results to time-dependent data generated by iterated pushforwards of a dynamical system, showing in Theorem~\ref{thm2} that identifiability persists under a Takens-type embedding assumption. As a consequence, Corollary~\ref{cor1} introduces metrics on spaces of diffeomorphisms and vector fields defined through finitely many measure-valued observations, providing a new finite-data framework for data fitting in distributional inverse problems.

More broadly, our results provide a unified framework for reconstruction problems arising in operator-theoretic dynamical systems, PDE inverse problems, and transport-based generative modeling. We showed that the same identifiability mechanism applies to the recovery of Perron--Frobenius and Koopman operators, to inverse problems for continuity, advection, Fokker--Planck, and advection--diffusion--reaction equations, and to transport-based generative models. The numerical experiments in Section~\ref{sec:numerics} further support the practical relevance of the theory by demonstrating recovery of vector fields and pushforward maps from finite distributional data.

From the perspective of machine learning, these results provide finite-data identifiability guarantees for transport- and operator-based models, clarifying when such measure-valued learning problems are well posed. This, in turn, lays groundwork for future study of stability, robustness, and generalization, since uniqueness is a prerequisite for meaningful control of error propagation and sensitivity to perturbations. More generally, by characterizing the number and type of distributions needed for unique recovery, our analysis may help guide the design of learning objectives, model classes, and data-collection strategies that enforce identifiability explicitly rather than implicitly.

Important directions for future work include relaxing the assumptions underlying the current theory. It will also be important to establish stability estimates for the proposed metrics, thereby quantifying how perturbations in the observed measures affect the recovery of the underlying diffeomorphism or vector field. Further avenues include extending the framework beyond the diffeomorphic setting, as well as developing a quantitative understanding of how identifiability is influenced by noise and finite-sample effects.

\section*{Acknowledgements}
J.~B.-G.~was supported in part by a fellowship award under contract FA9550-21-F-0003 through the National Defense Science and Engineering Graduate (NDSEG) Fellowship Program, sponsored by the Air Force Research Laboratory (AFRL), the Office of Naval Research (ONR) and the Army Research Office (ARO), and ONR under award N00014-24-1-2088. Y.~Y.~was supported in part by the National Science Foundation under award DMS-2409855 and by ONR under award N00014-24-1-2088.

\bibliographystyle{unsrt}
\bibliography{references}
\end{document}